\documentclass[letterpaper]{article} 
\usepackage{aaai25}  
\nocopyright
\usepackage{times}  
\usepackage{helvet}  
\usepackage{courier}  
\usepackage[hyphens]{url} 
\usepackage{graphicx} 
\usepackage{natbib} 
\usepackage{caption} 
\usepackage{algorithm}
\usepackage{algorithmic}

\usepackage{newfloat}
\usepackage{listings}
\usepackage{amsmath}
\usepackage{xcolor}
\usepackage{booktabs}
\usepackage{multirow}
\usepackage{amssymb}
\usepackage{amsthm}
\usepackage{enumitem}
\usepackage{subcaption}
\usepackage{float}

\newtheorem{proposition}{Proposition}
\newtheorem{assumption}{Assumption}

\newtheorem{theorem}{Theorem}

\newtheorem{problem}{Problem}

\DeclareCaptionStyle{ruled}{labelfont=normalfont,labelsep=colon,strut=off} 
\floatstyle{ruled}
\newfloat{listing}{tb}{lst}{}
\floatname{listing}{Listing}

\title{OpFlow: Learning Opportunity-Conditioned Choice Potentials for Robust OD Flow Prediction}

\author{
   Changjian Liu{\rm 1}, Yong Gao{\rm 1}\textsuperscript{*}, Yuqing Wang{\rm 1}, Leyi Su{\rm 1},\\
   Honglei Guo{\rm 2}, Zhiyang Wang{\rm 1}, Xiaoyu Wang{\rm 1}, Fan Zhang{\rm 1}} 
\affiliations{
    \textsuperscript{\rm 1} School of Earth and Space Sciences\\
    Peking University\\
    China, CN 100871 \\ 

    \textsuperscript{\rm 2} College of Artificial Intelligence\\
    Zhejiang University\\
    China, CN 310027 \\
}

\begin{document}

\maketitle

\begin{abstract}
Origin-destination (OD) flow prediction is central to urban analytics, yet deep models trained on raw counts remain vulnerable to distribution shift. The core problem is that raw count supervision cannot distinguish transferable choice mechanisms from environment-specific shortcuts. Raw OD count mixes two objects: how much demand an origin produces and how that demand is allocated across destinations. 
We argue that the transferable object is the exposure-to-choice law that maps spatial conditions to relative destination preferences. 
We propose \textbf{OpFlow}, a mechanism-constrained framework that learns row-centered choice potentials and reconstructs flows by combining the induced allocation with a separately calibrated origin scale. Under distribution shift, spatial exposures and the induced allocations are allowed to vary; what transfers is the conditional map from exposure states to relative choice potentials.
Theoretically, we characterize the identifiable row-centered potential and show that classical spatial interaction laws are restricted log-potential cases.
Controlled synthetic shifts and a real-world experiment show OpFlow improves robustness under environment shifts.
\end{abstract}

\section{Introduction}
\label{sec:intro}
Spatial flows, including commuting, migration, and freight movement, capture directed movements and interactions among spatial units \citep{fotheringham1989spatial, liu2024integrating, barbosa2018human}. Accurate origin-destination (OD) flow prediction is therefore fundamental to transportation planning, urban analytics, and policy evaluation \citep{zheng2014urban, rong2024interdisciplinary}. Recent spatiotemporal predictors and graph neural networks (GNNs) have achieved remarkable in-distribution (ID) accuracy by exploiting relational signals such as distance decay, historical co-occurrence, and static topology \citep{wang2019origin, shi2020predicting, han2022continuous}. However, such signals can play two different roles. When aligned with travel accessibility or opportunity structure, they support modeling; when they merely proxy environment-specific flow patterns, they become shortcuts \citep{geirhos2020shortcut, koh2021wilds}. For instance, geographic distance may encode stable travel impedance in one regime, but degenerate into a shortcut when opportunity distributions, infrastructure, or policies shift \citep{geurs2004accessibility}. The core challenge thus lies in extracting generalizable spatial interaction mechanisms from environment-dependent relational signals, rather than merely fitting historical patterns.

Classical spatial interaction theory provides a vital conceptual anchor. Models such as Gravity, Intervening-Opportunity, and Radiation explain OD flows through destination attraction, travel impedance, cumulative opportunities, and spatial competition \citep{stouffer1940intervening, fotheringham1983new, simini2012universal}. While differing in functional forms, they share a fundamental \emph{allocation view}: an origin first generates a demand scale, which is then distributed across destinations via a spatial choice structure. This decomposition is interpretable and transferable, but classical laws are often too rigid for data-rich urban systems. Conversely, deep models offer greater flexibility but typically optimize raw flows end-to-end \citep{simini2021deep}. We identify this target as the root cause of their out-of-distribution (OOD) fragility: raw count supervision mixes origin scale with destination allocation. Consequently, a model would reduce training loss by fitting context-specific outflow volumes or shortcut correlations without recovering the choice mechanism that should generalize.

To bridge this gap, we return to a micro-to-aggregate view implicit in spatial interaction theory. Each trip begins as a destination choice: a traveler at an origin compares available destinations, attracted by their opportunities, discouraged by travel impedance, and potentially intercepted by nearer alternatives. Aggregating such choices forms OD flows, while the stable object is not flow itself but the behavioral response that maps spatial exposure to relative destination preference. This motivates a target-level reformulation: robust flow prediction should learn how allocation responds to  exposure, rather than fit raw flows directly. We call this opportunity-conditioned choice potential learning: the model maps exposure signatures to choice potentials, i.e., relative destination scores whose softmax induces allocation. Under distribution shift, exposures and allocations may change, but the map is the mechanism intended to transfer.

We instantiate this principle in \textbf{OpFlow}, a mechanism-constrained framework for robust OD flow prediction. OpFlow learns a deployable exposure-to-potential map: it encodes spatial exposure (e.g. opportunities, impedance, intervening exposure, and local competition) into row-centered choice potentials whose softmax defines destination allocation. Origin production is predicted by a separate scale branch with gradient isolation, so OD flows are reconstructed through scale-allocation composition. To improve transfer under drift, it further imposes spatial monotonicity priors and optimizes robust allocation risk, encouraging the learned mechanism to remain stable under drift.

Our contributions are threefold:
i) We recast robust OD flow prediction as opportunity-conditioned choice-potential learning, identifying exposure-to-potential map as the transferable target.
ii) We connect the target to spatial interaction theory by characterizing potential identifiability, classical law special cases, and OOD risk decomposition.
iii) We develop \textbf{OpFlow}, an operator-structured neural model that combines mechanism-guided exposure operators, scale-isolated learning, and robust training for stable prediction.

\section{Related Work}
\label{sec:related}

\textbf{Spatial interaction and OD flow prediction.}
OD flow modeling has long been studied through spatial interaction theory, where aggregate flows are explained by origin production, destination attraction, travel impedance, intervening opportunities, and destination competition \citep{stouffer1940intervening, fotheringham1983new, fotheringham1989spatial, simini2012universal}. Gravity models emphasize OD mass and travel deterrence; intervening-opportunity models describe flows through the opportunities encountered before a destination is reached; radiation models derive mobility probabilities from opportunity thresholds and cumulative surrounding opportunities. These laws differ in calibration, transferability, and structural fidelity \citep{lenormand2016systematic}. Neural spatial interaction models further replace rigid-param functions with neural approximators \cite{fischer2001neural,fischer2002methodology, simini2021deep}.

Recent computational OD modeling has evolved from pairwise prediction to graph-structured and generative modeling. Early deep approaches introduced recurrent, convolutional, and attention-based architectures for short-term OD demand prediction, including urban rail OD forecasting under partial observability and stochastic OD matrix forecasting with graph convolutional recurrent networks \citep{hu2020stochastic,jiang2022deep, rong2024interdisciplinary}. Graph-based methods further represent regions and OD matrices as structured signals, using regional attributes, adjacency, distance, historical flows, and multi-perspective spatial dependencies to improve predictive flexibility \citep{wang2019origin,shi2020predicting,han2022continuous}. More recent work shifts from prediction to generation, including physics-inspired neural models and GAN or diffusion based OD network generators \citep{simini2021deep,rong2023gravity,rong2023complexity,rong2025large}. These methods mainly differ in how they parameterize, aggregate, or generate OD flows. Our focus is complementary: we ask which component of an OD flow should be learned under shift.

\textbf{OOD, causal, and disentangled graph learning.}
Generalization under distribution shift has received increasing attention in urban spatiotemporal prediction. Cross-city transfer methods adapt models trained in data-rich cities to data-scarce target cities by matching similar regions, reweighting source regions, or aligning spatiotemporal representations across domains \citep{wang2019crosscity,jin2022selective,fang2022transfer}. Recent benchmarks further show that spatiotemporal neural networks can degrade substantially under temporal or urban OOD splits \citep{wang2024evaluating}. These studies highlight the importance of transferability, but they usually keep the supervised target fixed and focus on aligning feature distributions, model parameters, graph structures, or source-target domains. We address a target-level issue specific to OD flows: a raw count mixes origin scale variation with destination choice structure. This perspective is related to generic OOD learning, including invariant risk minimization, risk extrapolation, and invariant graph representation learning \citep{arjovsky2019invariant,krueger2021out,li2022learning,wu2022discovering}; however, our invariant object is not a graph representation alone, but the exposure to choice mechanism that induces allocation.

\textbf{Choice, allocation, and opportunity exposure.}
Our formulation is also connected to discrete choice, entropy-based allocation, and accessibility theory. Luce's choice axiom and McFadden's conditional logit relate normalized choice probabilities to latent utilities, while entropy-based spatial interaction derives probabilistic flow allocation under aggregate constraints \citep{luce1959individual, mcfadden1972conditional, fotheringham1989spatial}. Accessibility theory emphasizes that destination availability and spatial opportunity structure shape human mobility and land-use interaction \citep{hansen1959accessibility, yang2025visitfrequency}. We build on these ideas to define the deployable learning target in OpFlow.

\section{Problem Setup and Theory}
\label{sec:theory}

\begin{figure*}[t]
    \centering
    \includegraphics[width=0.7\textwidth]{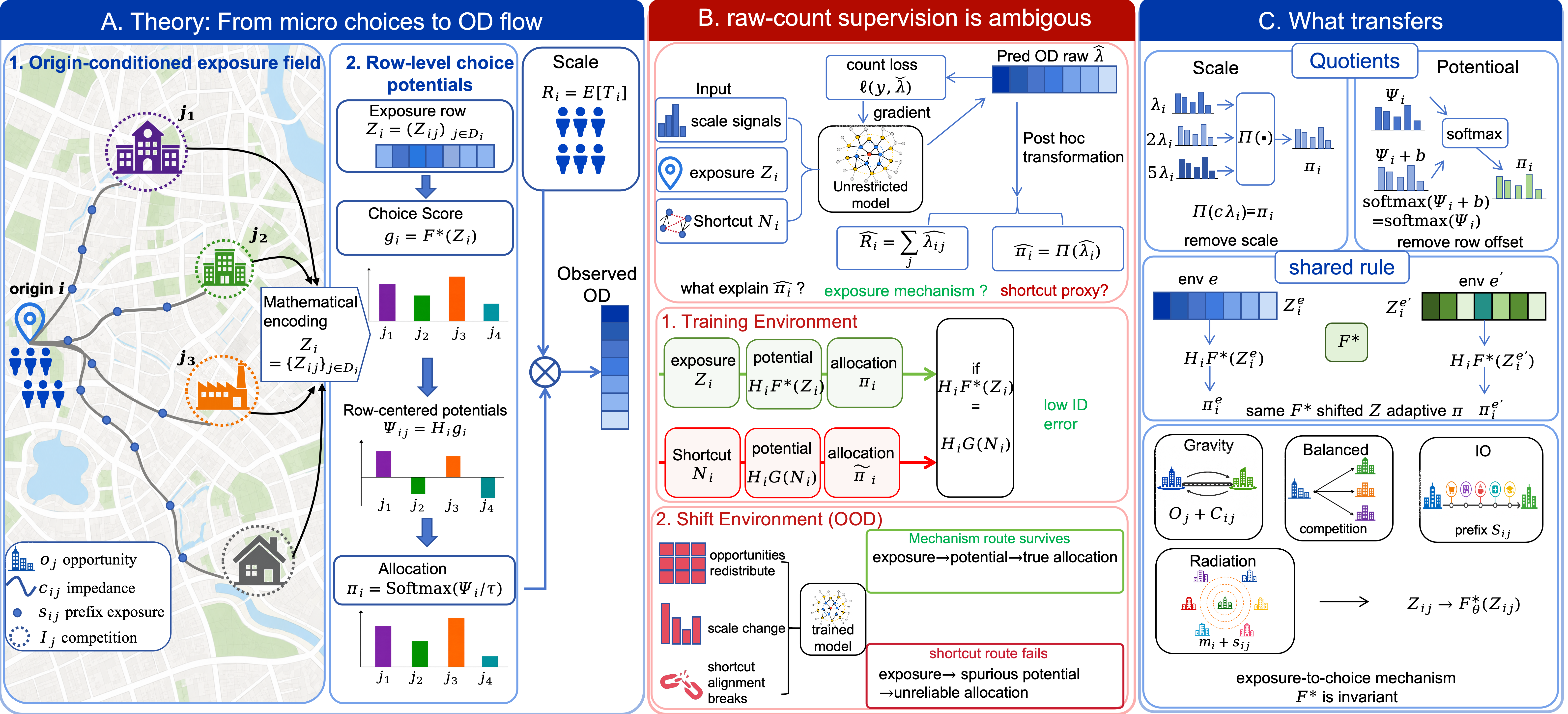}
    \caption{\textbf{From raw count to transferable mechanisms.}
    Raw OD counts make deployable mechanisms and spurious shortcuts indistinguishable ID. OpFlow instead targets the exposure-to-choice mechanism that remains stable under shift.
    }
    \label{fig:problem}
\end{figure*}

\begin{problem}[Robust OD flow prediction]
\label{prob:ood-od}
An environment $e\in\mathcal E$ (e.g., a city or time period) induces covariates $X^e$ and OD matrix $Y^e \in \mathbb{Z}_{\ge 0}^{n \times n}$ with conditional mean $\Lambda^e = \mathbb{E}[Y^e \mid X^e]$. Given training environments $\mathcal{E}_{\rm tr}$, robust prediction seeks $f$ minimizing the worst-case risk over unseen $e \in \mathcal{E}_{\rm te}$:
\begin{equation}
\label{eq:ood-risk}
\mathcal R^e(f)=\mathbb E_{X^e}
\left[
d\big(f(X^e), \Lambda^e\big)
\right],
\quad
\min_{f\in\mathcal F}
\sup_{e\in\mathcal E_{\rm te}}
\mathcal R^e(f).
\end{equation}
where $d$ is a distance metric (e.g., $L_1$ or KL divergence).
\end{problem}

\subsection{From Microscopic Choices to Macroscopic Flows}
\label{sec:micro-to-aggregate}
To understand the macroscopic OD flow, we zoom into the behavior of travelers. Let $T_i^e$ represent the random number of trips generated from origin $i$, with $R_i^e=\mathbb E[T_i^e]$. Each trip chooses a destination \(J_{ir}^e\), and the observed OD count is \(Y_{ij}^e=\sum_{r=1}^{T_i^e}\mathbf 1\{J_{ir}^e=j\}\). If \(J_{ir}^e\sim\mathrm{Categorical}(\pi_i^e)\), then:
\begin{equation}
\label{eq:aggregate-mean}
\underbrace{\lambda_{ij}^e}_{\text{OD row}} = \mathbb{E}[Y_{ij}^e] = \underbrace{R_i^e}_{\text{Production Scale}} \times \underbrace{\pi_{ij}^e}_{\text{Allocation Probability}}.
\end{equation}
An OD row contains two objects: the origin production scale \(R_i^e\), which controls how much flow leaves the origin, and the destination allocation \(\pi_i^e\), which controls where that flow goes. This decomposition allows us to model destination allocation directly as the underlying choice mechanism.

\subsection{The Opportunity-Conditioned Choice Mechanism}
\label{sec:choice-intensity}
We denote the input covariates as \(X^e_{ij} =
(x_i^{e,\mathrm{ori}},\allowbreak
 x_j^{e,\mathrm{dest}},\allowbreak
 x_{ij}^{e,\mathrm{pair}})\), comprising origin, destination and pairwise features. While these features are indexed locally, destination choice fundamentally depends on \emph{relative} exposure within an origin's entire choice set. To address this, we decouple the characterization of the choice environment from the behavioral response by introducing an operator $\mathcal O$:
\[Z_{ij}^{e}=\mathcal O(X^e)_{ij}=\left(
m_i^e,o_j^e,c_{ij}^e,s_{ij}^e,\ell_j^e\right).\]
Specifically, $m_i^e$, $o_j^e$, and $c_{ij}^e$ capture the origin context, destination opportunity, and pairwise impedance, respectively. Crucially, \(m^e_i\) represents observed origin heterogeneity that conditions the response to destination-varying exposure components; $s_{ij}^e$ and $\ell_j^e$ represent the relative choice-set competition and spatial neighborhood effects, computed via dedicated operators: \( s_{ij}^e=\mathcal O_s(o^e,c_i^e)_j,\) \(\ell_j^e=\mathcal O_\ell(o^e,\mathcal G^e)_j.\) Therefore, we model destination choice as an \emph{opportunity-conditioned intensity process}. Let $A_{ij}^\star > 0$ denote the latent intensity with which destination $j$ captures a trip from $i$. Intuitively, an effective opportunity is generated through a sequential filtering process: an origin-side aspiration channel must be active, the trip must survive intervening opportunities, and the destination must provide an acceptable match. Because these filters operate sequentially for a given latent state, the channel-specific capture rate factorizes across them, while travel impedance and local context act as external multiplicative modulators. Since the exact latent state is unobserved, the rates across alternative latent channels superpose via integration. 
The following theorem provides a random-utility foundation for this formulation ($e$ is omitted):

\begin{theorem}[Opportunity-conditioned Choice Intensity]
\label{thm:generalized-intensity}
For a trip generated at origin $i$, suppose each feasible destination $j\in\mathcal D_i$ competes to capture the trip by generating effective opportunities. 
Conditional on the exposure row $Z_i$, the first effective-opportunity arrival time from destination $j$ follows an independent exponential race with rate 
$A_{ij}^{\star}>0$, where
\begin{equation}
\label{eq:generalized-intensity}
A_{ij}^{\star}=
D(c_{ij}) C(\ell_j)
\int_{\Theta}
q_i(\theta) S(s_{ij};\theta) B(o_j;\theta)
\, d\mu(\theta).
\end{equation}
Here, $\theta$ indexes latent micro-opportunity states; 
$D,C,q_i,\allowbreak S,B$ denote travel deterrence, local agglomeration or competition,
origin aspiration, intervening-opportunity survival, and destination acceptance,
respectively. Then the destination allocation probability is
\[
\pi_{ij}^{\star}
=
\frac{A_{ij}^{\star}}{\sum_{k\in\mathcal D_i}A_{ik}^{\star}}
=
\operatorname{softmax}_j(g_i^\star/\tau),
\qquad
g_{ij}^{\star}=\tau\log A_{ij}^{\star}.
\]
Combined with Eq.~\eqref{eq:aggregate-mean}, this formulation directly yields the aggregate OD flow.
\end{theorem}

Theorem~\ref{thm:generalized-intensity} is OpFlow's micro-to-aggregate interface to model destination choice through log-intensities while preserving the row-normalized allocation structure. OpFlow instantiates a neural generalization of this form: the transferable mechanism is the exposure-to-choice law.
In the model, \(\mathcal O\) is learnable exposure construction, while \(F^\star\) is learned by neural deterrence, opportunity-benefit, intervening-exposure, and local-context operators.

\noindent\textbf{Targeting the Row-Centered Potential.}
\label{sec:row-centered-target}
Eq.~\eqref{eq:aggregate-mean} motivates learning the allocation mechanism; however, the potential that induces an allocation is identifiable only up to a row-wise additive offset.
Furthermore, since the allocation is governed by a Softmax function, the underlying log-potential $g_i^\star$ is inherently shift-invariant (i.e., adding a constant to all elements in a row does not change the probabilities). To ensure a unique and stable learning target, OpFlow does not predict raw counts or uncentered utilities. Instead, we explicitly target the \textbf{row-centered log-potential}:
\begin{equation}
\label{eq:row-centered-target}
H_i = I - \frac{1}{|\mathcal{D}_i^e|} \mathbf{1}\mathbf{1}^\top, \qquad \psi_i^e = H_i g_i^e.
\end{equation}

\subsection{Classical Spatial Laws as Special Cases}
\label{sec:classical-corollaries}

The generalized intensity framework (Eq.~\eqref{eq:generalized-intensity}) unifies classical spatial interaction models. By restricting the functional forms of deterrence ($D$), competition ($C$), survival ($S$), and benefit ($B$), the temperature-scaled log-potentials ($g_{ij}/\tau = \log A_{ij}$) of well-known spatial interaction laws emerge as special cases. We summarize these mappings in Table~\ref{tab:classical-laws} with derivations in Appendix. Crucially, while these classical laws are constrained by rigid functional forms, OpFlow neuralizes the log-potential, capturing complex, non-linear spatial interactions inaccessible to rigid parametric models.

\begin{table}[t]
\centering
\caption{Classical models as restricted cases.}
\label{tab:classical-laws}
\renewcommand{\arraystretch}{1.3}
\small
\begin{tabular}{@{}p{1.6cm} p{3.2cm} p{2.8cm}@{}}
\toprule
\textbf{Model} & \textbf{Restricted Components} & \textbf{Log-Potential } $\boldsymbol{g_{ij}/\tau}$ \\ \midrule
Gravity 
& $D=e^{-\beta c_{ij}}, B=o_j^\alpha$ \newline $S=1, C=1$ 
& $\alpha\log o_j - \beta c_{ij}$ \\[6pt]

Comp. Dest. 
& $D=e^{-\beta c_{ij}}, B=o_j^\alpha$ \newline $C=e^{\gamma\ell_j}$ 
& $\alpha\log o_j - \beta c_{ij} + \gamma\ell_j$ \\[6pt]

Interv. Opp.  
& $S=e^{-\eta s_{ij}}$ \newline $B=1-e^{-\eta o_j}$ 
& $\log(1-e^{-\eta o_j}) - \eta s_{ij}$ \\[6pt]

Radiation
& Threshold-based CDF \newline integration 
& $\log o_j - \log(M_i{+}s_{ij})$ \newline $- \log(M_i{+}s_{ij}{+}o_j)$ \\ 
\bottomrule
\end{tabular}
\end{table}

\subsection{Invariant Choice Mechanisms}
\label{sec:raw-count-ambiguity}

\paragraph{The Pitfall of Raw-Count Supervision.}
The failure of raw-count supervision is fundamentally a target entanglement problem. Crucially, the same covariate can support a true mechanism when used through the structured operator path (e.g., distance as physical impedance $c_{ij}\mapsto D(c_{ij})$), yet become a shortcut when an unrestricted predictor uses it as an environment-specific statistical proxy. Because the variance in $R_i^e$ typically dominates the variance in $\pi_i^e$ by orders of magnitude, the optimization is hijacked by scale fitting. Consequently,  flexible predictors exploit environment-specific scale shortcuts, completely bypassing the true structural mechanism $X_i^e \mapsto Z_i^e \xrightarrow{F^\star} g_i^{e,\star} \xrightarrow{H_i} \psi_i^{e,\star} \mapsto \pi_i^{e,\star}$. This yields deceptively low ID risk but guarantees OOD failure: without decoupling, scale and preference variations remain fatally entangled, rendering the model untransferable.

\paragraph{Invariant Behavioral Response.}
To resolve the shortcut ambiguity, we distinguish between volatile macroscopic statistics and stable microscopic mechanisms. OpFlow does not assume that raw OD flows or allocation probabilities are invariant; such variables inherently undergo covariate shift as spatial contexts change. Instead, our core premise is that the \emph{behavioral response law}: the conditional mapping from exposure states to choice intensities is assumed to remain shared across environments after conditioning on exposure states. Crucially, Assumption~\ref{assump:exposure-stability} ensures that under covariate shift with adequate exposure coverage and bounded residual drift, the allocation probabilities adapt \emph{equivariantly} through $F^\star$, structurally decoupling the stable behavioral mechanism from volatile environmental scales.

\begin{assumption}[Invariant Exposure-to-Choice Mechanism]
\label{assump:exposure-stability}
While the marginal distributions of exposures $Z_i^e$, origin scales $R_i^e$ shift arbitrarily across $e$, the true row-centered log-potential admits the decomposition:
\begin{equation}
\label{eq:mechanism-decomp}
\psi_i^{e,\star} = H_iF^\star(Z_i^e) + \Delta_i^e,
\end{equation}
where $F^\star$ is the shared mechanism, and $\Delta_i^e$ represents bounded, environment-specific structural drift unexplained by $Z_i^e$. Since both sides are row-centered, $\Delta_i^e$ is also taken to be row-centered.
\end{assumption}

\subsection{OOD Generalization Bounds}
\label{sec:ood-consequences}

Let $P_{\rm tr}$ and $P_{\rm te}$ denote the training and deployment distributions over exposure rows $Z_i$, and assume a bounded density ratio $dP_{\rm te}/dP_{\rm tr}\le \rho$. For clarity, we suppress the environment index when no ambiguity arises. Under Assumption~\ref{assump:exposure-stability}, define the allocation induced by the shared mechanism and the true allocation with residual drift as
\(
\pi_{0,i} = \operatorname{softmax}\left(H_iF^\star(Z_i)/\tau\right)\),\(\pi_i^\star = \operatorname{softmax}\left((H_iF^\star(Z_i)+\Delta_i)/\tau\right).
\)
The learned allocation is
\(
\widehat{\pi}_i^\theta = \operatorname{softmax}\left(H_iF_\theta(Z_i)/\tau\right),
\)
where $F_\theta(Z_i)$ is the model-predicted uncentered score row. We define the training allocation risk against the shared mechanism as
\(
R_{\rm tr}^{\rm alloc}(\theta) = \mathbb{E}_{\rm tr} \left[ \operatorname{KL}\left( \pi_{0,i} \| \widehat{\pi}_i^\theta \right) \right].
\)

\begin{proposition}[Row-level OOD allocation bound]
\label{prop:ood-drift}
Under Assumption~\ref{assump:exposure-stability} and $dP_{\rm te}/dP_{\rm tr}\le \rho$, the expected deployment allocation error satisfies
\[
\mathbb{E}_{\rm te} \left[ \left\| \pi_i^\star - \widehat{\pi}_i^\theta \right\|_1 \right] \le \underbrace{ \sqrt{2\rho R_{\rm tr}^{\rm alloc}(\theta)} }_{\text{transferable error}} + \underbrace{ \mathbb{E}_{\rm te} \left[ \frac{\sqrt{|\mathcal{D}_i|}}{\tau} \left\|\Delta_i\right\|_2 \right] }_{\text{residual drift}}.
\]
\end{proposition}

This bound separates two sources of OOD error. The first term is controlled by training allocation risk and exposure coverage, while the second term measures environment-specific drift not explained by the exposure signature. Thus allocations need not be invariant across environments; they change with $Z_i$, but the response law through $F^\star$ is the transferable object.

Finally, using the mean decomposition in Eq.~\eqref{eq:aggregate-mean}, the row-level count error decomposes into scale error and allocation error. If $\widehat{\lambda}_i=\widehat{R}_i\widehat{\pi}_i^\theta$, $\lambda_i^\star=R_i^\star\pi_i^\star$, and $R_i^\star\le R_{\max}$, then
\begin{equation}
\label{eq:row-count-bound}
\begin{aligned}
\mathbb{E}_{\rm te}
\left[\left\|\widehat{\lambda}_i-\lambda_i^\star\right\|_1\right]
&\le \mathcal{R}_{\rm te}^{\rm scale} \\
&\quad + R_{\max}\,\mathbb{E}_{\rm te}
\left[\left\|\widehat{\pi}_i^\theta-\pi_i^\star\right\|_1\right].
\end{aligned}
\end{equation}
where
\(
\mathcal{R}_{\rm te}^{\rm scale} = \mathbb{E}_{\rm te} \left[ \left| \widehat{R}_i - R_i^\star \right| \right].
\)
Equation~\eqref{eq:row-count-bound} is used only to connect allocation transfer to OD-flow reconstruction; the scale branch is specified in the Scale-Isolated Reconstruction subsection below.

\section{Methodology}
\label{sec:method}

\begin{figure}[t]
    \centering
    \includegraphics[width=\linewidth]{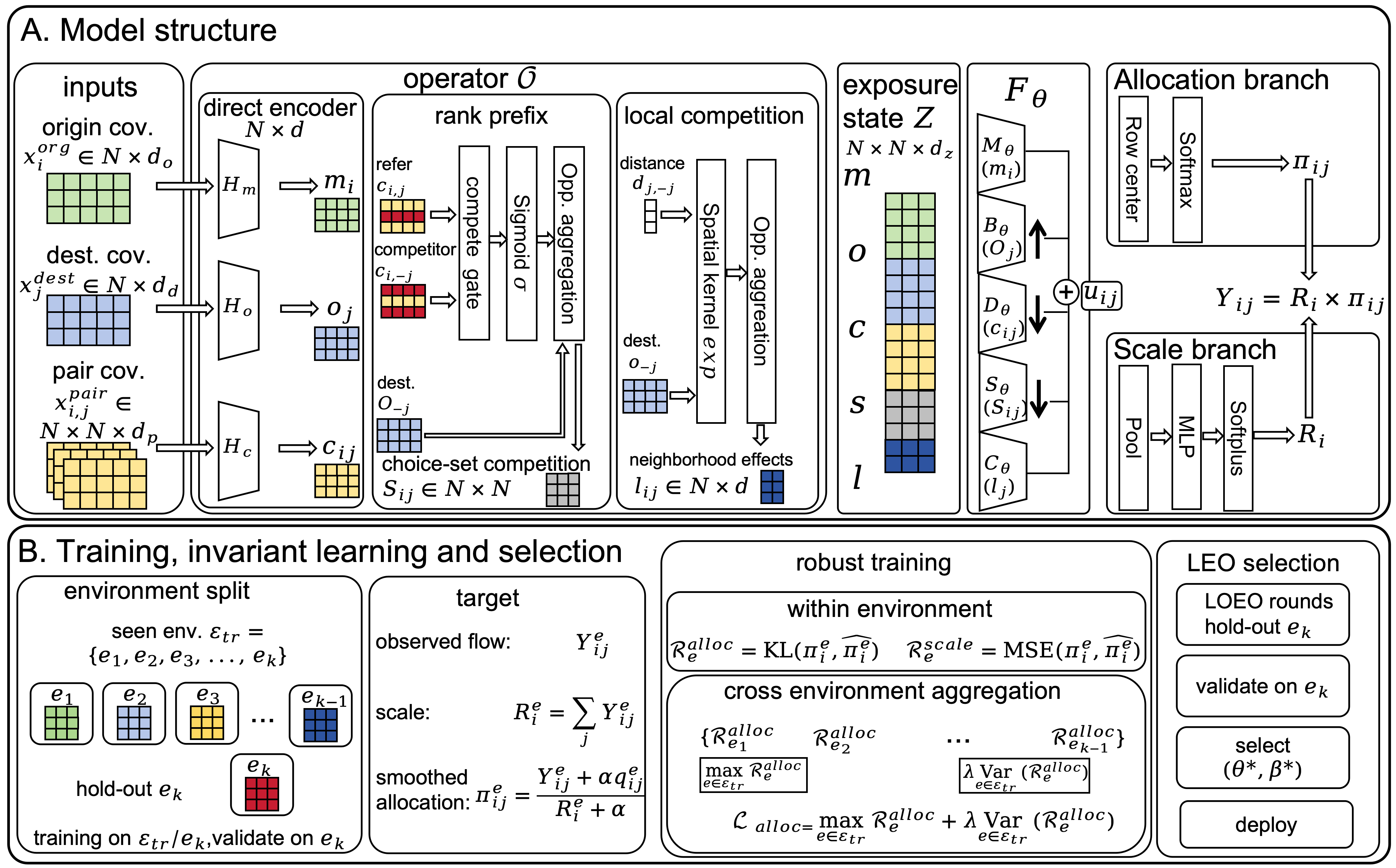}
    \caption{
    Overview of the proposed method.
    }
    \label{fig:opflow_model}
\end{figure}

\subsection{Overview}
As shown in Fig.~\ref{fig:opflow_model}, OpFlow implements the theoretical framework by structurally decoupling the macroscopic OD flow into origin scale and destination allocation (Eq.~\eqref{eq:aggregate-mean}). Instead of using an unrestricted black-box network that is prone to shortcut learning, it parameterizes the invariant exposure-to-choice mechanism $F^\star$ (Assumption~\ref{assump:exposure-stability}) via structured neural spatial-interaction operators. These operators strictly follow the opportunity-conditioned intensity law (Theorem~\ref{thm:generalized-intensity}), ensuring the learned potential remains grounded in transferable behavioral mechanisms.

\subsection{Operator-Structured Choice Potential}
Given the input $X^e$, OpFlow first applies an exposure construction operator $\mathcal O$ to build exposure state $Z_{ij}^e$. $\mathcal O$ contains both direct neural encodings and differentiable spatial exposure layers. The direct encoders map origin, destination, and pairwise covariates into exposure representations:
\(m_i^e=H_m(X_i^{e,\mathrm{ori}})\), \(o_j^e=H_o(X_j^{e,\mathrm{dest}})\), \(c_{ij}^e=H_c(X_{ij}^{e,\mathrm{pair}})\), where $H_m$, $H_o$, and $H_c$ are MLPs. The rank-prefix layer computes an intervening-opportunity field: \(s_{ij}^e=\sum_{k\in\mathcal D_i^e,\;k\neq j}o_k^e\,\sigma\left(\frac{c_{ij}^e-c_{ik}^e}{T}\right)\), where $T$ is a soft ranking temperature, \(\sigma\) is the sigmoid function. The local opportunity-field layer constructs destination competition or agglomeration: \(\ell_j^e=\sum_{k\neq j}w_{jk}^e o_k^e\), \(w_{jk}^e=\exp\left(-\frac{d_{jk}^e}{\omega}\right).\) Here $d_{jk}^e$ is the distance, and $w_{jk}^e$ is the exponential kernel. We then compute bounded origin gates \((a_i,b_i,r_i,\gamma_i)=\mathbf 1+\rho\tanh(G_\theta(m_i^e)),\) which modulate the destination-varying exposure responses.

Then we approximate the temperature-scaled choice score $u_{ij}^\star = \tau\log A_{ij}^\star$ using an operator-structured neural network $F_\theta$. $F_\theta$ decomposes the uncentered choice score $u_{ij}^e$ into additive neural operators corresponding to the components of Theorem~\ref{thm:generalized-intensity}:
\begin{equation}
\label{eq:neural-operators}
\small
u_{ij}^e = a_i B_\theta(o_j^e) - b_i D_\theta(c_{ij}^e) - r_i S_\theta(s_{ij}^e) + \gamma_i C_\theta(\ell_j^e) + \eta I_\theta(m_i^e,z_{ij}^e),
\end{equation}
where $D_\theta, B_\theta, S_\theta, C_\theta$ are neural approximators for travel deterrence, destination benefit, intervening-opportunity survival, and local competition, respectively, and \(I_\theta(m_i^e,z_{ij}^e)\) is an interaction term for remaining origin-destination effects.

\noindent\textbf{Physical Monotonicity Prior.} 
We regularize the scalar exposure operators with a soft monotonicity penalty, encouraging deterrence and intervening exposure to decrease utility and opportunity to increase it:
\begin{equation*}
\begin{split}
\Omega_{\rm op} = \mathbb{E} \Bigg[ & \mathrm{ReLU}\left(-\frac{\partial D_\theta}{\partial c}\right)^2 + \mathrm{ReLU}\left(-\frac{\partial S_\theta}{\partial s}\right)^2 \\
& + \mathrm{ReLU}\left(-\frac{\partial B_\theta}{\partial o}\right)^2 \Bigg].
\end{split}
\end{equation*}

\noindent\textbf{Row-Centered Allocation.} 
To eliminate the origin-specific nuisance scale and satisfy the identifiability condition (Eq.~\eqref{eq:row-centered-target}), we apply row-centering to the uncentered scores:
\(\hat{\psi}_{\theta,i}^e = H_i u_{\theta, i}^e.\)
The destination allocation is then obtained via a temperature-scaled softmax.

\subsection{Scale-Isolated Reconstruction}
To predict the volatile origin production scale $R_i^e$, we employ a separate scale branch. The scale is predicted using origin-level attributes $m_i^e$ and a stop-gradient aggregated summary of the choice potentials:
\begin{equation}
\label{eq:scale-branch}
\hat{R}_i^e = \mathrm{softplus}\left( S_\beta\big(m_i^e, \mathrm{sg}(\mathrm{Pool}_{j \in \mathcal{D}_i^e}(u_{ij}^e))\big) \right),
\end{equation}
where $\mathrm{sg}(\cdot)$ blocks gradients from the scale loss to the allocation potential, enforcing structural decoupling. The final expected OD flow is reconstructed as their product Eq.~\eqref{eq:aggregate-mean}.

\subsection{Environment-Robust Training}
To optimize for the worst-case OOD risk (Problem~\ref{prob:ood-od}) and bound the transferable error (Proposition~\ref{prop:ood-drift}), we employ an environment-robust training objective. First, the empirical allocation target is smoothed to ensure full support: $\tilde{\pi}_{ij}^e = (Y_{ij}^e + \alpha q_{ij}^e) / (\sum_k Y_{ik}^e + \alpha)$. The environment-wise allocation risk is defined as $\mathcal{R}_e^{\rm alloc} = \frac{1}{|\mathcal{I}_e|} \sum_{i \in \mathcal{I}_e} \mathrm{KL}(\tilde{\pi}_i^e \parallel \hat{\pi}_i^e)$.

\noindent\textbf{Allocation Loss.}
We minimize the worst-case risk across seen environments, augmented with a variance penalty to discourage environment-specific shortcut learning:
\begin{equation}
\label{eq:alloc-loss}
\mathcal{L}_{\rm alloc} = \max_{e \in \mathcal{E}_{\rm tr}} \mathcal{R}_e^{\rm alloc} + \lambda_{\rm vrex} \mathrm{Var}_{e \in \mathcal{E}_{\rm tr}}(\mathcal{R}_e^{\rm alloc}).
\end{equation}

\noindent\textbf{Scale and Calibration Losses.} 
The scale branch is optimized via a worst-case log-space MSE to handle long-tail distributions: $\mathcal{L}_{\rm scale} = \max_{e} \frac{1}{|\mathcal{I}_e|} \sum_i [\log(1+\hat{R}_i^e) - \log(1+R_i^e)]^2$. To physically calibrate the count reconstruction without corrupting the allocation simplex, we apply a detached count loss: $\mathcal{L}_{\rm count} = \mathbb{E}_{e, i} [ d_{\rm count}(\hat{R}_i^e \mathrm{sg}(\hat{\pi}_i^e), Y_i^e) ]$.

\noindent\textbf{Total Objective and model selection.} 
The complete training objective integrates the mechanism constraints and robust optimization:
\begin{equation}
\label{eq:total-loss}
\mathcal{L} = \mathcal{L}_{\rm alloc} + \lambda_{\rm scale}\mathcal{L}_{\rm scale} + \lambda_{\rm count}\mathcal{L}_{\rm count} + \lambda_{\rm op}\Omega_{\rm op}.
\end{equation}
For deployment, we select models by cross-environment validation. The training environments are partitioned into \(K\) folds \(\{\mathcal E_k\}_{k=1}^K\). Each candidate configuration is trained on \(\mathcal E_{\rm tr}\setminus \mathcal E_k\) and evaluated on \(\mathcal E_k\) using the held-out stable allocation risk. The configuration with the lowest aggregated held-out risk is retrained on all training environments, and the resulting \(F_\theta,S_\beta\) are deployed. LEO selection is the special case \(K=|\mathcal E_{\rm tr}|\).

\section{Experiments}
\subsection{Synthetic Experiments}

\paragraph{DGP and interventions.}
We sample \(N=80\) spatial units on the 2D square \([0,1]^2\) to form an OD system, where each $i$ has a candidate set $\mathcal D_i$ of its nearest $K_i \sim \mathcal{N}(25,5)$ destinations. Per seed, 14 environments (8/2/4 for train/val/test) share spatial structures but have different opportunity fields $o_j^e$ (from a Gaussian-mixture field), scales, allocations and counts. The oracle log potential is designed according to Eq.~\eqref{eq:generalized-intensity}: \(g_{ij}^{e,*}=  a_i O_j^e - b_i C_{ij} - \beta_s S_{ij}^e+\beta_\ell L_j^e+\Delta_{ij}^e\), with standardized components,  Gaussian drift $\Delta_i^e$, and context-aware coefficients $a_i, b_i$. Counts are sampled as \(Y_{ij}^e\sim\mathrm{Poisson}(R_i^e\pi_{ij}^{e,*})\). Thus, $Y_{ij}^e$ confounds scale and allocation, requiring OpFlow to reconstruct exposures from noisy proxies.
We then ask four questions: 
\textbf{Q1} Does invariant mechanism learning remove origin scale nuisance that misleads raw-count supervision? what are its limitations?
\textbf{Q2} Can the exposure-to-choice law generate new allocations when opportunities are redistributed?
\textbf{Q3} Can structured exposure modeling prevent models from using accessibility as a shortcut for opportunity when their training correlation breaks?
\textbf{Q4} Which components are load-bearing under OOD stress?

\noindent
\begin{minipage}[t]{\columnwidth} 
    \centering
    \includegraphics[width=\textwidth]{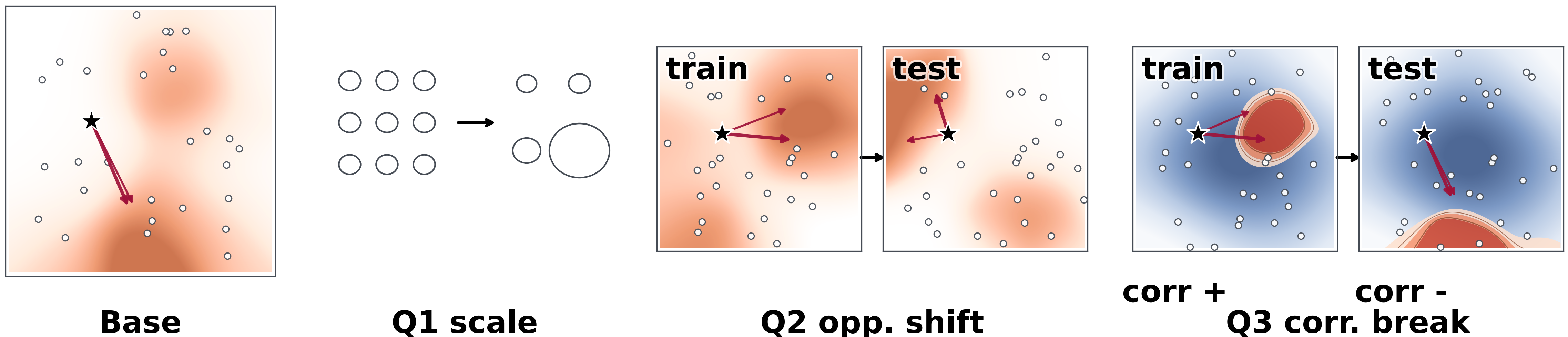}
    \captionof{figure}{Synthetic interventions used to evaluate target-level robustness, opportunity redistribution, exposure-correlation shift, and component ablations.}
    \label{fig:synthetic_intervention}
\end{minipage}

\paragraph{Metrics and baselines.}
We evaluate from allocation, mechanism, and flow count recovery. For allocation, we use KL divergence and CPC (Common Part of Commuters \citep{lenormand2012universal}) calculated by rows. For mechanism recovery, we report the MSE of the potential (PMSE). For counts, we report RMSE and MAE. All metrics average 30 seeds with 95\% CIs.
Baselines include: (1) MLP probes for specific supervisions/assumptions (\textsc{RawCountMLP}, \textsc{PairAllocMLP}, \textsc{ScaleAwareMLP}); (2) Domain generalization methods (\textsc{GroupDRO}, \textsc{VREx}, \textsc{IRM}, \textsc{DANN}) on the \textsc{PairAllocMLP} backbone, reporting \textsc{BestDG} selected by val KL; (3) Classical spatial interaction models, reporting \textsc{BestClassical}. Full results are in the Appendix.

\paragraph{Results.} Fig.~\ref{fig:synthetic_intervention} illustrates the interventions for Q1--Q4. Q1 shows a target-level failure, we increase train/validation origin scale severity while keeping all other factors fixed. This intervention tests whether raw count supervision is hijacked by the high-variance scale and therefore learns a worse allocation rule. Fig.~\ref{fig:q1_res} shows that \textsc{RawCountMLP} deteriorates as scale heterogeneity increases, whereas \textsc{PairAllocMLP} and OpFlow remain nearly flat. This shows the separation of Eq.~\eqref{eq:aggregate-mean} is necessary toward robust prediction. 
Q2 tests opportunity redistribution by fitting models at the baseline severity and evaluating them under shifted destination-opportunity fields. Since the oracle allocation changes with the opportunity landscape, the goal is not to remain perfectly flat but to preserve a low allocation error. OpFlow achieves the lowest Row-KL across severities; at the strongest shift, it improves over \textsc{BestDG-AllocMLP} by 23.7\%. \textsc{ShareStable} and \textsc{RawCountMLP} degrade sharply, indicating that fixed shares and raw-count fitting fail to adapt to redistributed opportunities. The small gain of \textsc{BestDG-AllocMLP} over its DG components further suggests that generic DG regularization cannot replace the opportunity-conditioned structure used by OpFlow.
Q3 evaluates an exposure-correlation shift. Models are trained and validated under positive opportunity-accessibility correlation, and then evaluated when this correlation is reversed. This tests whether a model can separate opportunity attraction from travel impedance rather than relying on their correlation. As shown in Fig.~\ref{fig:q3_res}, generic allocation models deteriorate under the reversed correlation: \textsc{BestDG-AllocMLP} increases from 0.265 to 0.331 in Row-KL, and closely tracks the ordinary allocation MLP family. In contrast, OpFlow achieves the lowest error at every severity and reaches 0.156 Row-KL at the strongest reversal, reducing the error of \textsc{BestDG-AllocMLP} by 52.9\% and outperforming it. \textsc{RawCountMLP} and \textsc{ShareStable} degrade more substantially, confirming that count-level and fixed-share shortcuts are brittle under correlation shifts. These results indicate that OpFlow’s structured exposure channels help separate admissible components whose marginal association changes across environments.
Q4 ablates the main OpFlow components at the Q2 and Q3 stress endpoints. The clearest effect comes from the rank-prefix exposure: removing it substantially increases Row-KL under opportunity redistribution and exposure-correlation reversal. The local opportunity field has a smaller but clear role under the Q3 correlation stress, where its removal increases Row-KL. These ablations suggest that OpFlow’s gains mainly come from explicitly modeling rank-prefix competition and, under correlation shift, local destination context.

\begin{figure}[htbp]
    \centering
    \begin{subfigure}[b]{0.48\linewidth}
        \centering
        \includegraphics[width=\linewidth]{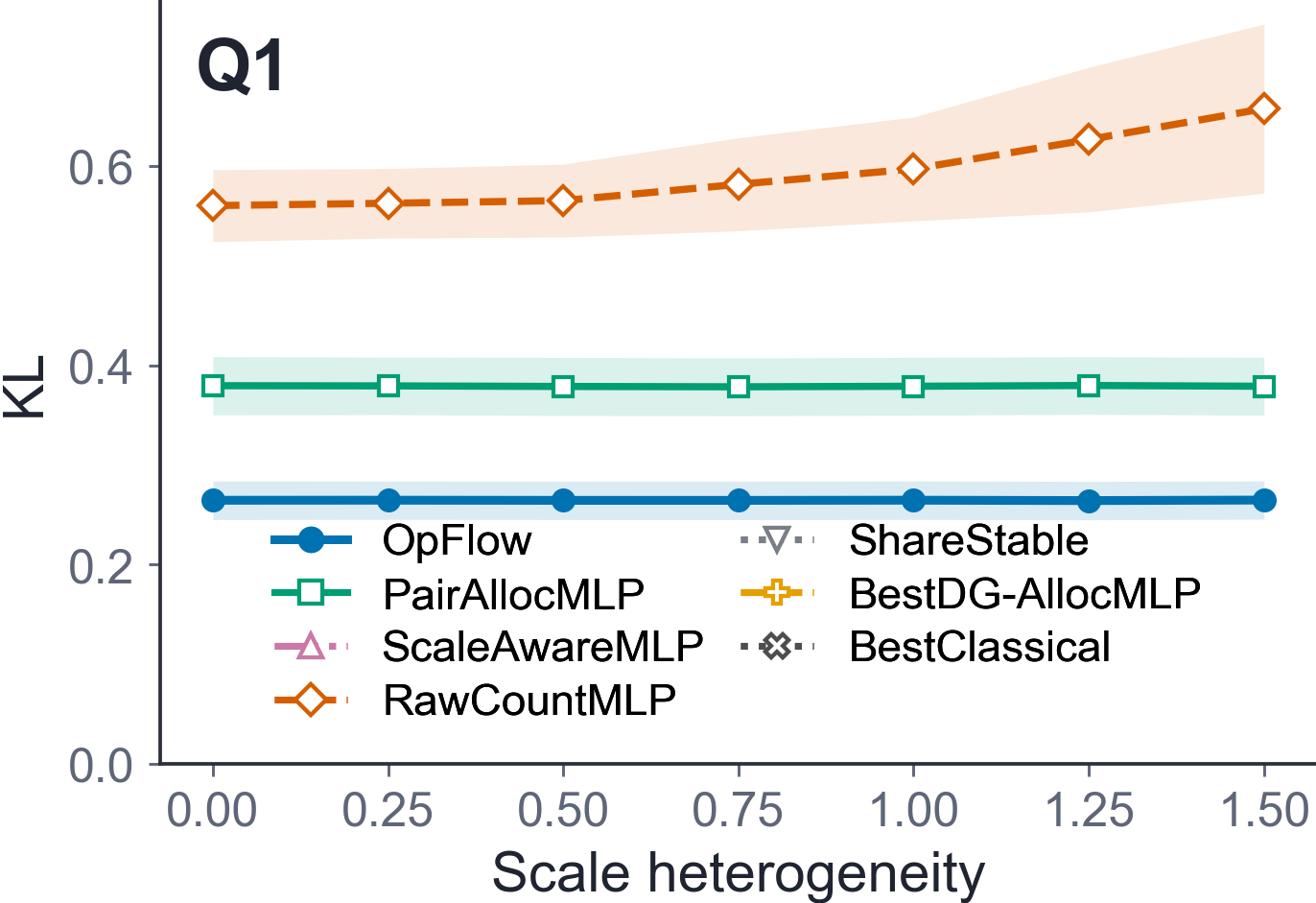}
        \caption{Q1 result}
        \label{fig:q1_res}
    \end{subfigure}
    \hfill 
    \begin{subfigure}[b]{0.48\linewidth}
        \centering
        \includegraphics[width=\linewidth]{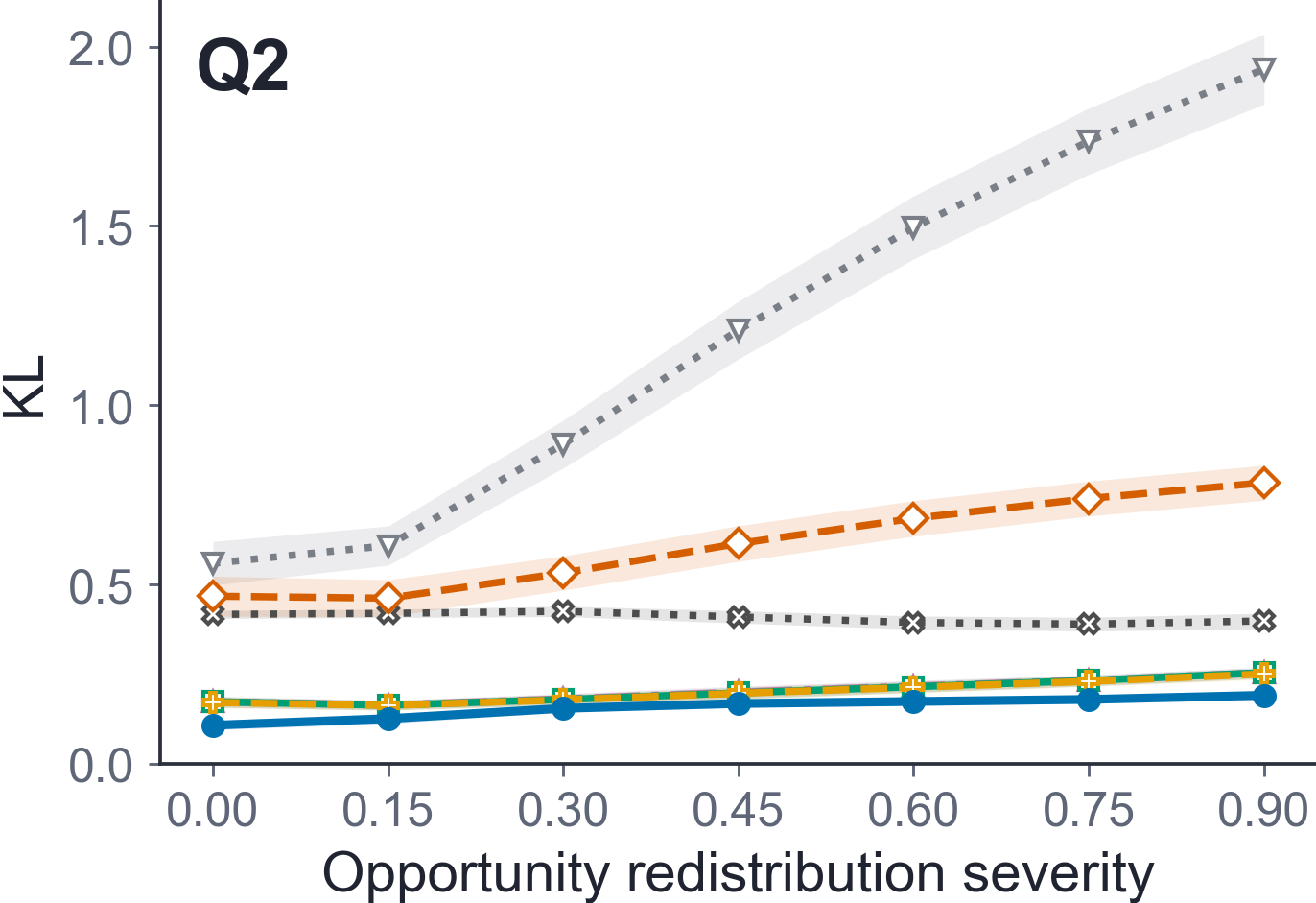}
        \caption{Q2 result}
        \label{fig:q2_res}
    \end{subfigure}
    \vspace{0em}
    \begin{subfigure}[b]{0.48\linewidth}
        \centering
        \includegraphics[width=\linewidth]{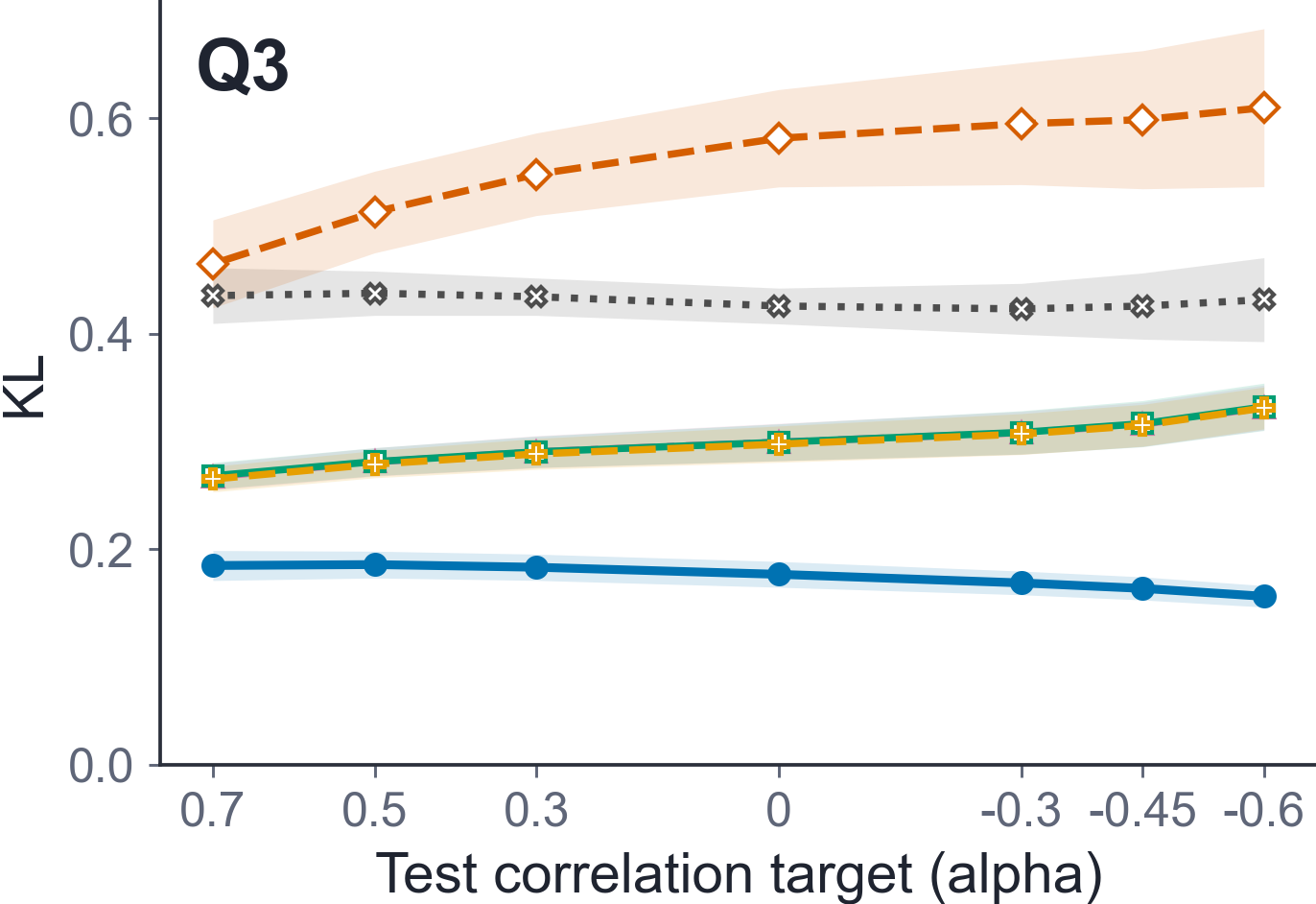}
        \caption{Q3 result}
        \label{fig:q3_res}
    \end{subfigure}
    \hfill
    \begin{subfigure}[b]{0.48\linewidth}
        \centering
        \includegraphics[width=\linewidth]{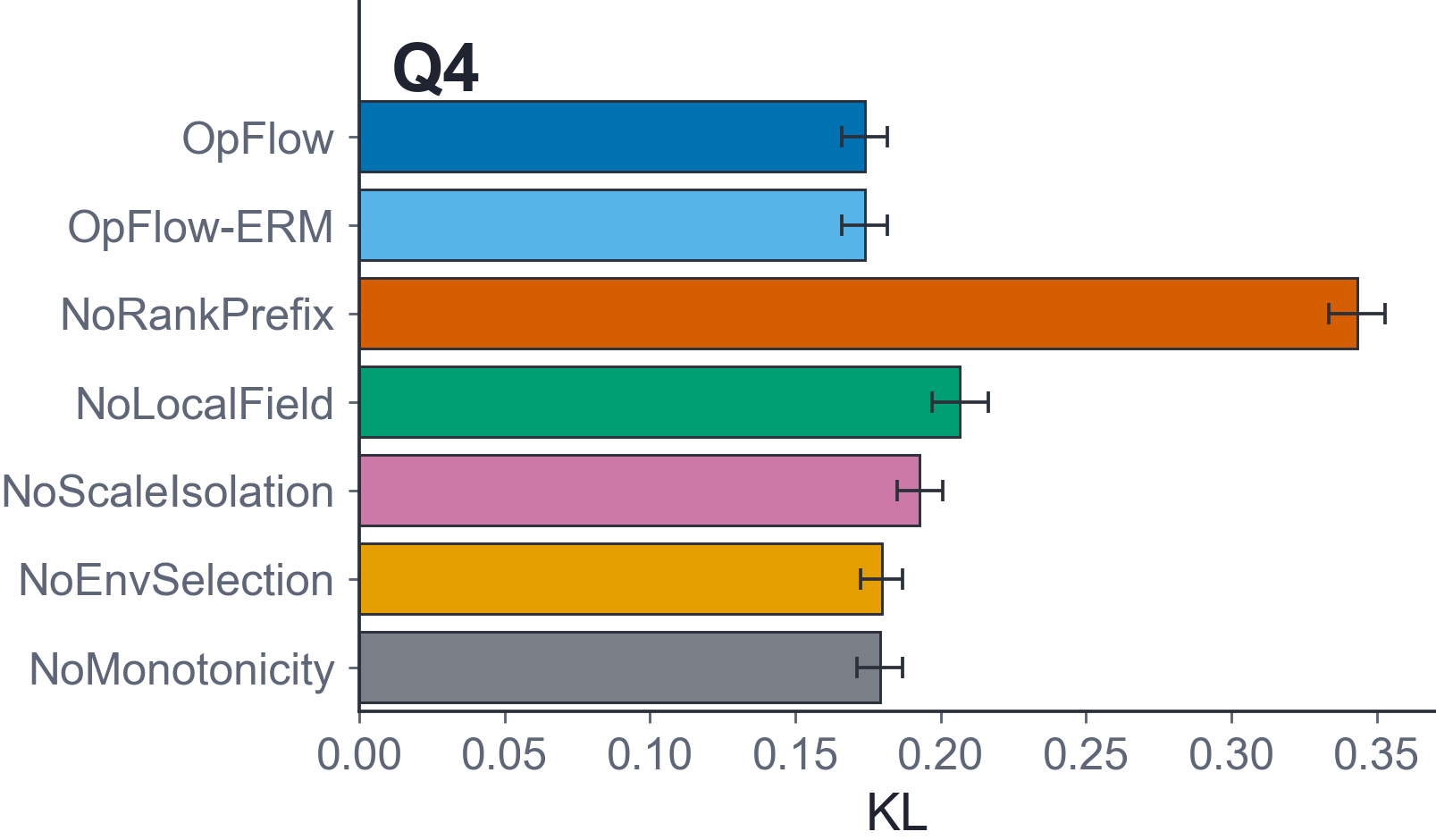}
        \caption{Q4 result}
        \label{fig:q4_res}
    \end{subfigure}
        
    \caption{Synthetic OOD results across four interventions. (a) Origin-scale heterogeneity mainly harms raw-count supervision. (b) Opportunity redistribution tests whether models adapt allocation to new destination-opportunity fields. (c) Exposure-correlation reversal tests whether opportunity attraction is separated from travel impedance. (d) Ablations identify load-bearing exposure channels.}
    \label{fig:feature_grid}
\end{figure}

\subsection{Real-world Experiments}
\label{sec:realworld}

\begin{table*}[t]
\centering
\scriptsize
\setlength{\tabcolsep}{3.5pt}
\renewcommand{\arraystretch}{1.15}
\caption{
\textbf{Real-world cross-county generalization.}
We report Row-KL, Row-CPC, and Log-RMSE on each split.
Values are formatted as $\text{KL} / \text{CPC} / \text{Log-RMSE}$.
Lower is better for KL and Log-RMSE; higher is better for CPC.
}
\label{tab:realworld-split}

\begin{tabular}[t]{@{}l ccc ccccc @{}}
\toprule
\multicolumn{4}{c}{\textit{Classical}} & \multicolumn{5}{c}{\textit{Generic ML}} \\
\cmidrule(lr){2-4} \cmidrule(lr){5-9}
Split & Gravity & Int. Opp. & Rad. & RF & GBRT & GM & Pair & Scale \\
\midrule
RandSP & 0.545/0.613/1.468 & 8.590/0.260/3.132 & 0.999/0.538/1.582 & 1.364/0.420/2.129 & 0.810/0.521/1.655 & 2.842/0.426/2.414 & 0.399/0.684/1.222 & 0.438/0.676/1.193 \\
GeoSP  & 0.575/0.604/1.419 & 8.729/0.252/3.060 & 1.088/0.524/1.540 & 1.235/0.440/2.028 & 0.775/0.532/1.634 & 2.850/0.414/2.486 & 0.392/0.684/1.201 & 0.407/0.672/1.180 \\
OppSP  & 0.692/0.571/1.336 & 8.747/0.248/2.826 & 1.487/0.480/1.534 & 0.961/0.495/1.743 & 0.781/0.537/1.663 & 2.929/0.395/2.612 & 0.468/0.682/1.257 & 0.481/0.693/1.222 \\
Avg.   & 0.604/0.596/1.408 & 8.689/0.253/3.006 & 1.191/0.514/1.552 & 1.186/0.452/1.967 & 0.789/0.530/1.651 & 2.874/0.412/2.504 & 0.413/0.684/1.220 & 0.441/0.685/1.201 \\
\bottomrule
\end{tabular}
\hfill
\begin{tabular}[t]{@{}l ccccc c @{}}
\toprule
\multicolumn{5}{c}{\textit{Deep/Graph}} & \multicolumn{1}{c}{\textit{Ours}} \\
\cmidrule(lr){2-6} \cmidrule(lr){7-7}
Split & DGM & NetGAN & DiffODGen& WeDAN & GMEL & \textsc{OpFlow} \\
\midrule
RandSP & 0.656/0.545/4.337 & 0.595/0.615/1.470 & 0.676/0.611/1.623 & 0.645/0.609/1.629 & 0.435/0.674/1.299 & \textbf{0.322/0.710/1.088} \\
GeoSP  & 0.712/0.530/4.447 & 0.749/0.572/1.652 & 0.761/0.579/1.705 & 0.711/0.582/1.689 & 0.514/0.655/1.322 & \textbf{0.370/0.686/1.14} \\
OppSP  & 0.792/0.510/4.631 & 0.800/0.578/1.498 & 0.711/0.592/1.416 & 0.740/0.589/1.443 & 0.558/0.635/1.386 & \textbf{0.438/0.670/1.04} \\
Avg.   & 0.720/0.529/4.472 & 0.715/0.588/1.540 & 0.716/0.594/1.581 & 0.699/0.593/1.587 & 0.502/0.655/1.336 & \textbf{0.377/0.688/1.09} \\
\bottomrule
\end{tabular}
\end{table*}

\paragraph{Overview of dataset and splits.}
We evaluate OpFlow on a large commuting OD flow benchmark~\citep{rong2025large}, which contains commuting flows among 3,333 county-level areas in the United States. Each county is subdivided into multiple census blocks, between which commuting flows occur. Since we focus exclusively on intra-county commuting patterns, each county is treated as an independent environment. For each county, we construct an OD matrix where the origin and destination covariates consist of socio-demographic attributes (population structure of a region based on age, gender, income, education) and POI counts (counts of different POI categories) as origin and destination covariates, and use network distance and spatial adjacency as pairwise covariates.
We evaluate cross-county generalization using three spatial partition protocols, where each county is treated as one environment. \textsc{RandSP} randomly partitions counties into training, validation, and test sets, and serves as the base ID-style reference. 
\textsc{GeoSP} evaluates geographic OOD generalization by holding out one Census-division block at a time and training on counties from the remaining divisions. 
\textsc{OppSP} evaluates opportunity structure OOD generalization by holding out counties with atypical input-side opportunity concentration. 
Specifically, for each county, we compute a normalized Herfindahl concentration score\(C_e^{\mathrm{opp}}=\frac{\sum_v (p_v^e)^2 - 1/n_e}{1-1/n_e},\) where \(n_e\) is the number of tracts in the county. Counties in the lower and upper tails of \(C_e^{\mathrm{opp}}\) are held out for testing, corresponding to unusually dispersed and unusually concentrated opportunity structures.  Figure~\ref{fig:split-visualization} visualizes the three split protocols.

\noindent
\begin{minipage}[t]{\columnwidth} 
    \centering
    \includegraphics[width=0.9\textwidth]{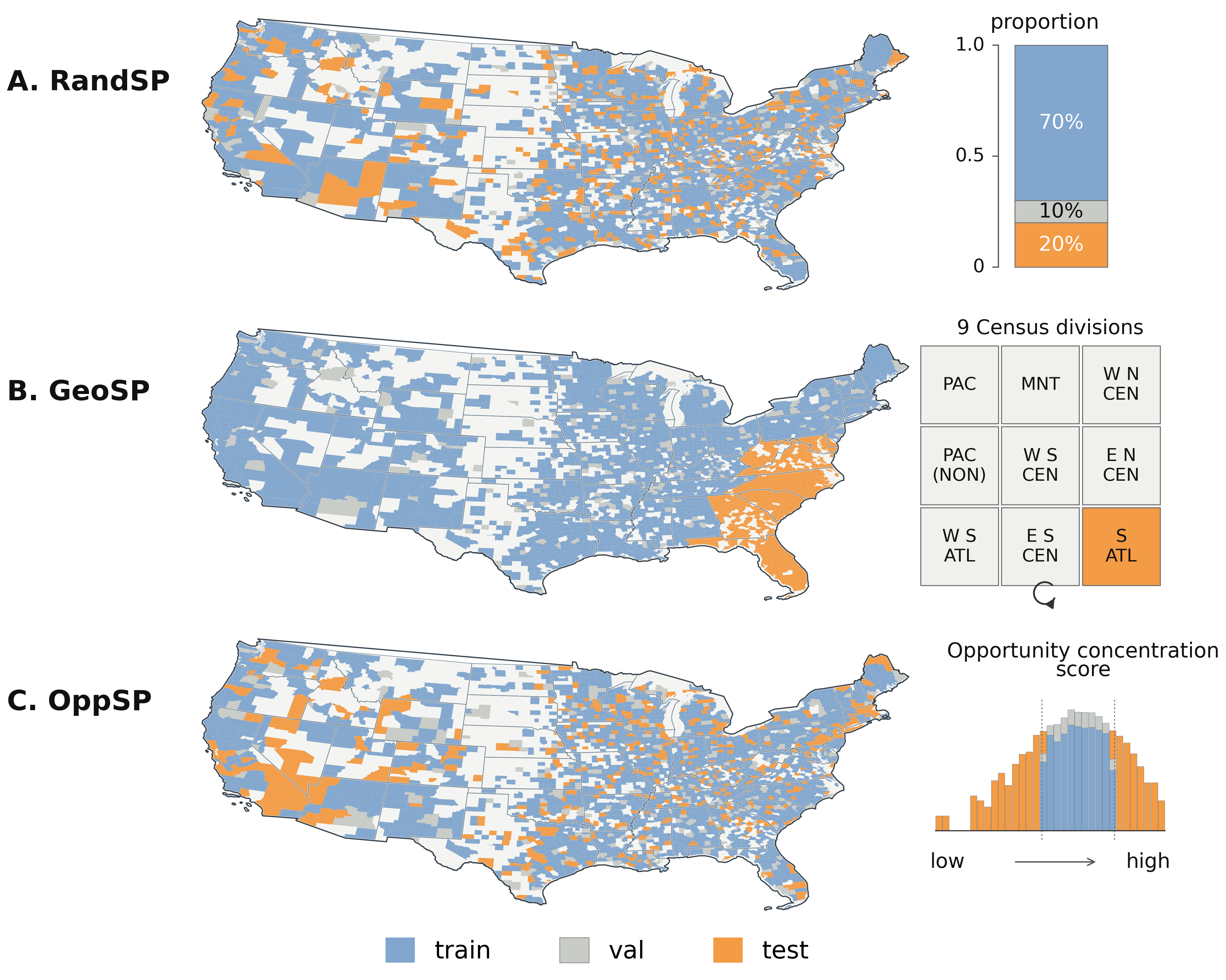}
    \captionof{figure}{Overview of the environment split protocols.}
    \label{fig:split-visualization}
\end{minipage}

\paragraph{Results.}
Table~\ref{tab:realworld-split} shows that OpFlow gives the best overall cross-county generalization. It achieves the lowest Row-KL and Log-RMSE on every split and the best average Row-CPC. On average, OpFlow reduces Row-KL from 0.410, achieved by the strongest generic baseline \textsc{PairMLP}, to 0.377, and reduces Log-RMSE from 1.20, achieved by \textsc{ScaleMLP}, to 1.09. The baseline patterns are also informative. Classical spatial-interaction models encode useful geographic priors, with Gravity being the strongest among them, but their fixed-form potentials remain systematically misspecified under cross-county heterogeneity. Generic ML models improve over classical laws by learning nonlinear associations, yet tree-based and global models still struggle to recover row-wise allocation structure. \textsc{PairMLP} and \textsc{ScaleMLP} are the closest competitors, indicating that allocation supervision and scale awareness are both important, but their higher Log-RMSE or Row-KL shows that these ingredients alone do not yield a transferable exposure mechanism. Deep/graph OD baselines capture richer OD dependencies, with GMEL being the strongest in this family, but they remain less robust than OpFlow because they do not explicitly separate origin scale from destination allocation or encode opportunity-conditioned choice structure. The consistent gains on both \textsc{GeoSP} and \textsc{OppSP} suggest that OpFlow does not merely fit county-specific flow magnitudes, but learns a more transferable allocation mechanism.

\section{Conclusion}
We tackled robust OD flow prediction by shifting to a target-level perspective. We observed that raw OD counts conflate origin scale with destination allocation, leading models to learn spurious relational shortcuts rather than transferable choice mechanisms. OpFlow resolves this by decoupling origin production from destination allocation via an exposure-to-choice mapping. This not only bridges neural prediction with classical spatial-interaction theory but also allows us to decompose OOD error into transferable allocation error, residual structural drift, and scale error, clarifying the conditions for robust transfer. Empirical evaluations show that OpFlow improves robustness under the predicted shifts, such as origin-scale variation and opportunity redistribution. Ultimately, our results highlight a crucial paradigm shift: robust spatial flow prediction requires learning the right supervised objective, not merely building stronger graph architectures. Future work will extend this mechanism-aligned framework to dynamic and intervention-driven OD systems.

\bibliography{references}

\newpage

\appendix

\section{Notations}
\label{app:notation}
This appendix summarizes the core notation used in OpFlow, see table~\ref{tab:notation-all}.

\begin{table*}[t]
\centering
\small
\caption{Summary of Core Notation.}
\label{tab:notation-all}
\renewcommand{\arraystretch}{1.25}
\begin{tabular}{@{}p{2.0cm} p{3.5cm} p{12cm}@{}}
\toprule
\textbf{Symbol} & \textbf{Name} & \textbf{Description} \\
\midrule
\multicolumn{3}{@{}l}{\textbf{Core Problem and Theoretical Notation}} \\
\midrule
$\mathcal{E}, e$ & Environment index & Spatial regimes (e.g., cities, time periods); $\mathcal{E}_{\rm tr}$ and $\mathcal{E}_{\rm te}$ denote train/deployment sets. \\
$\mathcal{D}_i^e$ & Candidate destination set & Available destinations for origin $i$ in  $e$. \\
$P_{\rm tr}, P_{\rm te}$ & Exposure distributions & Training and deployment distributions over exposure states $Z_i$. \\
$\rho$ & Density ratio bound & Upper bound on $dP_{\rm te}/dP_{\rm tr}$, measuring covariate shift severity. \\
$R_i^e, R_{\max}$ & Origin scale \& bound & Expected trip production from $i$; $R_{\max}$ is the uniform upper bound for OOD proofs. \\
$\pi_i^e, \lambda_{ij}^e$ & Allocation \& Mean flow & Destination choice probability vector; $\lambda_{ij}^e = R_i^e \pi_{ij}^e$. \\
$Z_{ij}^e$ & Exposure state & Mechanistic features: impedance $c_{ij}^e$, opportunity $o_j^e$, etc. \\
$s_{ij}^e$ & Intervening exposure & Rank-prefix cumulative opportunities: \(s_{ij}^e=\sum_{k\in\mathcal D_i^e,\;k\neq j}o_k^e\,\sigma\left(\frac{c_{ij}^e-c_{ik}^e}{T}\right).\) \\
$\ell_j^e$ & Local competition & Destination side agglomeration or competition \(\ell_j^e=\sum_{k\neq j}w_{jk}^e o_k^e\), \(w_{jk}^e=\exp\left(-\frac{d_{jk}^e}{\omega}\right).\) \\
$A_{ij}^\star$ & Latent choice intensity & True positive intensity; decomposed into $D, C, S, B$ kernels. \\
$g_{ij}^\star$ & Log-potentials & Uncentered log-intensity. \\
$F^\star, \Delta_i^e$ & Invariant law \& drift & Shared exposure-to-choice mechanism; bounded environment residual drift. \\
$\tau$ & Choice temperature & Softmax temperature controlling allocation sharpness. \\
\addlinespace[1ex]
\midrule
\multicolumn{3}{@{}l}{\textbf{Model Architecture and Training Notation}} \\
\midrule
$\Psi_\theta, S_\beta$ & Network branches & Stable potential network, Scale predictor \\
$\widehat{R}_i^e, \widehat{\pi}_{ij}^e$ & Scale \& Allocation & Predicted origin production and Softmax-normalized destination allocation. \\
${\rm sg}(\cdot)$ & Stop-gradient & Operator (e.g., \texttt{detach()}) blocking backpropagation to enforce decoupling. \\
$\mathcal{L}_{\rm alloc}$ & Allocation loss & Label-smoothed KL divergence for optimizing $\widehat{\pi}_i$. \\
$\mathcal{L}_{\rm scale}$ & Scale loss & Log-space MSE for optimizing $\widehat{R}_i$ to handle long-tail distributions. \\
$\mathcal{L}_{\rm count}$ & Count calibration & Reconstruction loss using detached allocations to physically calibrate scale. \\
$\tilde{\pi}_{ij}$ & Smoothed target & Full-support empirical target: $(Y_{ij} + \alpha q_{ij}) / (\sum_k Y_{ik} + \alpha)$. \\
LEO & Validation protocol & Leave-Environment-Out selection based solely on stable allocation KL risk. \\
\bottomrule
\end{tabular}
\end{table*}

\section{Proof of Theorem~\ref{thm:generalized-intensity}}
\label{app:proof-choice-intensity}

\begin{proof}
We suppress the environment superscript for clarity. The proof proceeds in three steps: first, we derive the opportunity-conditioned intensity from a sequential filtering mechanism; second, we show that competing opportunity arrivals induce a Softmax allocation; third, we connect the individual allocation to aggregate OD counts.

\paragraph{Step 1: From sequential filtering to the intensity.}
Fix an origin $i$, a destination $j$, and a latent micro-opportunity channel $\theta\in\Theta$. We model the generation of an effective opportunity as a sequential filtering process. A potential interaction through channel $\theta$ is realized only if it simultaneously satisfies four independent conditions: origin aspiration, travel impedance, intervening-opportunity survival, and destination acceptance.

Formally, let $dN_{ij\theta}^{\mathrm{eff}}(t)$ denote the counting measure of effective capture events. In an infinitesimal time interval $[t, t+dt]$ and latent cell $d\mu(\theta)$, the probability of a baseline channel activation is $q_i(\theta)d\mu(\theta)dt + o(dt)$. Conditional on this activation, the interaction is independently retained with probabilities $D(c_{ij})$, $C(\ell_j)$, $S(s_{ij};\theta)$, and $B(o_j;\theta)$, corresponding to travel deterrence, local context, intervening survival, and destination acceptance, respectively. 

By the chain rule for independent sequential filters, the joint survival probability is the product of the marginal probabilities. Equivalently, in the framework of point processes, this constitutes the independent thinning of a rare-event process. Thus, the expected number of effective events is
\begin{align*}
    &\mathbb{E}\!\left[dN_{ij\theta}^{\mathrm{eff}}(t)\mid Z_i\right]
= \\
&D(c_{ij})C(\ell_j)q_i(\theta)S(s_{ij};\theta)B(o_j;\theta)
\,d\mu(\theta)dt
+
o(dt).
\end{align*}
This implies that the channel-specific effective events form a Poisson process with intensity density (the Radon-Nikodym derivative with respect to $\mu \otimes \mathrm{Lebesgue}$):
\[
a_{ij}(\theta)
=
D(c_{ij})C(\ell_j)q_i(\theta)S(s_{ij};\theta)B(o_j;\theta).
\]

While distinct latent channels $\theta$ represent mutually exclusive pathways for an individual trip, at the aggregate level they act as independent sources of effective opportunities. By the superposition theorem for independent Poisson processes, the total destination-level capture rate $A_{ij}^{\star}$ is obtained by integrating the channel-specific intensity densities over the latent space $\Theta$:

\begin{align*}
    A_{ij}^{\star}&=
\int_\Theta a_{ij}(\theta)d\mu(\theta)
\\& =
D(c_{ij}) C(\ell_j)
\int_{\Theta}
q_i(\theta)S(s_{ij};\theta)B(o_j;\theta)
\,d\mu(\theta).
\end{align*}

This yields Eq.~\eqref{eq:generalized-intensity}.

\paragraph{Step 2: From competing capture rates to Softmax allocation.} Let $N_{ij}(t)$ denote the total number of effective opportunities generated by destination $j$ for a trip from origin $i$ before time $t$. By the rare-event construction above, conditional on $Z_i$, $N_{ij}(t)$ is a Poisson process with rate $A_{ij}^{\star}$. Therefore, the first effective-opportunity arrival time
\[
T_{ij}
=
\inf\{t\ge 0:N_{ij}(t)>0\}
\]
satisfies
\[
\Pr(T_{ij}>t)
=
\Pr(N_{ij}(t)=0)
=
\exp(-A_{ij}^{\star}t).
\]
Thus
\[
T_{ij}\sim\mathrm{Exp}(A_{ij}^{\star}).
\]

The traveler selects the destination whose effective opportunity arrives first:
\[
J_i=\arg\min_{j\in\mathcal D_i}T_{ij}.
\]
Because the destination-specific races are conditionally independent given $Z_i$,
\[
\Pr(J_i=j)
=
\Pr(T_{ij}<T_{ik},\forall k\neq j).
\]
Using the density of $T_{ij}$ and the survival functions of the competing races,
\[
\Pr(J_i=j)
=
\int_0^\infty
A_{ij}^{\star}e^{-A_{ij}^{\star}t}
\prod_{k\neq j}
e^{-A_{ik}^{\star}t}
\,dt.
\]
Therefore,
\[
\Pr(J_i=j)
=
\int_0^\infty
A_{ij}^{\star}
\exp\left(
-t\sum_{k\in\mathcal D_i}A_{ik}^{\star}
\right)
dt
=
\frac{A_{ij}^{\star}}
{\sum_{k\in\mathcal D_i}A_{ik}^{\star}}.
\]
Introducing a scale parameter $\tau>0$ and defining
\[
g_{ij}^{\star}
=
\tau\log A_{ij}^{\star},
\]
we obtain
\[
\frac{A_{ij}^{\star}}
{\sum_{k\in\mathcal D_i}A_{ik}^{\star}}
=
\frac{\exp(g_{ij}^{\star}/\tau)}
{\sum_{k\in\mathcal D_i}\exp(g_{ik}^{\star}/\tau)}
=
\operatorname{softmax}_j(g_i^\star/\tau).
\]
Hence
\[
\pi_{ij}^{\star}
=
\operatorname{softmax}_j(g_i^\star/\tau).
\]

\paragraph{Step 3: From individual choices to aggregate OD means.}
For the aggregate OD count, let $M_i$ denote the total number of trips generated at
origin $i$. The flow from $i$ to $j$ is
\[
Y_{ij}
=
\sum_{r=1}^{M_i}
\mathbf 1\{J_{ir}=j\}.
\]
Conditional on $M_i$ and $Z_i$,
\[
\mathbb E[Y_{ij}\mid M_i,Z_i]
=
M_i\pi_{ij}^{\star}.
\]
Taking expectation over $M_i$ gives
\[
\lambda_{ij}^{\star}
=
\mathbb E[Y_{ij}\mid Z_i]
=
\mathbb E[M_i\mid Z_i]\pi_{ij}^{\star}
=
R_i\pi_{ij}^{\star},
\]
where $R_i=\mathbb E[M_i\mid Z_i]$ is the expected total production scale at origin
$i$. This completes the proof.
\end{proof}

\paragraph{Why the intensity factorizes multiplicatively.}
The product form in Eq.~\eqref{eq:generalized-intensity} is not an arbitrary parametric choice; it follows rigorously from the logic of sequential search. For a channel-specific opportunity to capture a trip, several filters must succeed jointly: the origin must have demand for the latent channel, the destination must be reachable under travel impedance, the traveler must not be intercepted by closer opportunities, and the destination must ultimately provide an acceptable match. Since these filters are arranged as an ``AND'' sequence, their conditional survival probabilities or acceptance rates combine multiplicatively. By contrast, different latent states $\theta$ represent alternative micro-opportunity channels. Alternative channels combine via superposition, which mathematically manifests as the integral over $\Theta$. Thus, multiplication arises \emph{within} a channel due to sequential filtering, while integration arises \emph{across} channels due to latent superposition.

\paragraph{Mechanistic interpretation of the intensity.}
The aggregate rate $A_{ij}^{\star}$ represents the total rate at which destination $j$ can generate an effective opportunity for origin $i$. The latent state $\theta$ indexes micro-opportunity or matching channels, such as activity purpose, job type, service category, preference threshold, or unobserved destination quality. For a fixed channel $\theta$, the term
\[
q_i(\theta)S(s_{ij};\theta)B(o_j;\theta)
\]
combines three sequential components: the origin's aspiration for channel $\theta$, the probability that the traveler survives intervening opportunities before reaching $j$, and the destination's ability to provide an acceptable opportunity under that same channel. Furthermore, the pairwise deterrence $D(c_{ij})$ attenuates the spatial exposure between $i$ and $j$, while $C(\ell_j)$ adjusts the destination-side capture rate according to the local spatial context (e.g., agglomeration or competition). Integrating over $\Theta$ then aggregates all latent opportunity channels into a unified destination-level capture rate.

\paragraph{Connection to random utility theory.}
The same construction is strictly equivalent to a log-additive random-utility representation. At the channel level, we can define the systematic log-intensity as
\begin{equation*}
\begin{aligned}
v_{ij}(\theta)
= \tau \big[
&\log D(c_{ij}) + \log C(\ell_j) + \log q_i(\theta) \\
&+ \log S(s_{ij};\theta) + \log B(o_j;\theta)
\big].
\end{aligned}
\end{equation*}
It follows that
\begin{equation*}
\begin{aligned}
\exp(v_{ij}(\theta)/\tau)
&= D(c_{ij})C(\ell_j)q_i(\theta)
   S(s_{ij};\theta)B(o_j;\theta) \\
&= a_{ij}(\theta).
\end{aligned}
\end{equation*}
Thus, the multiplicative capture rate in intensity space corresponds exactly to an additive decomposition of systematic utility in log space. This aligns with the standard economic interpretation: separable utility components add before exponentiation, while the corresponding choice intensities multiply.

Moreover, the exponential-race formulation recovers the classical Gumbel random-utility model as an equivalent representation. Since $T_{ij}\sim\mathrm{Exp}(A_{ij}^{\star})$, let us define the utility as $U_{ij}=-\tau\log T_{ij}$. Then, for any $u \in \mathbb{R}$,
\[
\Pr(U_{ij}\le u)
=
\Pr(T_{ij}\ge e^{-u/\tau})
=
\exp(-A_{ij}^{\star}e^{-u/\tau}).
\]
This can be algebraically rewritten as
\[
\Pr(U_{ij}\le u)
=
\exp\left[
-\exp\left(
-\frac{u-\tau\log A_{ij}^{\star}}{\tau}
\right)
\right],
\]
which is precisely the cumulative distribution function of a Gumbel random variable with location $\tau\log A_{ij}^{\star}$ and scale $\tau$. Hence, we can write
\[
U_{ij}
=
\tau\log A_{ij}^{\star}
+
\varepsilon_{ij},
\qquad
\varepsilon_{ij}\sim\mathrm{Gumbel}(0,\tau).
\]
Because the transformation $-\tau\log(\cdot)$ is strictly decreasing, minimizing the arrival time is equivalent to maximizing the utility:
\[
\arg\min_jT_{ij}
=
\arg\max_jU_{ij}.
\]
Therefore, the Gumbel random-utility formulation is not imposed as an ad hoc starting point; rather, it is a natural, equivalent consequence of competing effective-opportunity arrivals.

\paragraph{Identifiability of the row-centered potential.}
The softmax allocation depends exclusively on the relative capture rates within the same origin row:
\[
\pi_{ij}^{\star}
=
\frac{A_{ij}^{\star}}{\sum_kA_{ik}^{\star}}.
\]
For any origin-specific positive constant $\kappa_i>0$, scaling the rates yields
\[
\frac{\kappa_iA_{ij}^{\star}}
{\sum_k\kappa_iA_{ik}^{\star}}
=
\frac{A_{ij}^{\star}}
{\sum_kA_{ik}^{\star}}.
\]
Thus, the absolute race speed (or baseline intensity) of an origin row is fundamentally unidentifiable from allocation data alone. In log space, this multiplicative ambiguity translates into a row-wise additive offset:
\[
\log(\kappa_iA_{ij}^{\star})
=
\log A_{ij}^{\star}+\log\kappa_i.
\]
Consequently, the identifiable behavioral target is the row-centered log-potential, defined as
\[
\psi_i^\star
=
H_ig_i^\star,
\qquad
H_i
=
I_{|\mathcal D_i|}
-
\frac{1}{|\mathcal D_i|}
\mathbf 1\mathbf 1^\top,
\]
where $g_{ij}^\star = \tau \log A_{ij}^\star$. This row-centering operation elegantly removes the unidentifiable origin-specific race speed while preserving the relative destination preferences that govern allocation. Crucially, it also decouples the transferable spatial allocation mechanism from the origin production scale $R_i$, which independently controls the total volume of flow leaving the origin.

\section{Classical Spatial Interaction Laws}
\label{app:classical-laws}

This section demonstrates that several classical spatial interaction laws arise as restricted cases of the proposed opportunity-conditioned intensity in Eq.~\eqref{eq:generalized-intensity}. Recall that the temperature-scaled potential is defined as $g_{ij}=\tau\log A_{ij}$. For the softmax allocation, it is the unscaled log-intensity, $g_{ij}/\tau = \log A_{ij}$, that directly determines the relative choice probabilities. All expressions below are understood up to row-wise additive constants (i.e., terms depending only on the origin $i$), which cancel out during row-centering.

Starting from the general form,
\begin{equation}
\label{eq:general-A-app}
A_{ij} = D(c_{ij})C(\ell_j) \int_\Theta q_i(\theta)S(s_{ij};\theta)B(o_j;\theta) \, d\mu(\theta),
\end{equation}
we consider the restricted case without latent heterogeneity. Let the state space be a singleton $\Theta=\{\theta_0\}$ with $q_i(\theta_0)=1$. The integral collapses, yielding:
\begin{align}
\frac{g_{ij}}{\tau} &= \log A_{ij} \notag \\
&= \log D(c_{ij}) + \log C(\ell_j) + \log S(s_{ij}) + \log B(o_j).
\end{align}

\paragraph{Gravity model.}
The production-constrained gravity model is recovered by choosing:
\[
D(c_{ij})=\exp(-\beta c_{ij}),
B(o_j)=o_j^\alpha,
S(s_{ij})=1, C(\ell_j)=1,
\]
with $\alpha,\beta>0$. This yields $A_{ij}=o_j^\alpha\exp(-\beta c_{ij})$, and thus:
\begin{align}
\frac{g_{ij}}{\tau} &= \alpha\log o_j - \beta c_{ij}.
\end{align}
This shows that the gravity model is a log-potential model characterized by destination attraction and travel deterrence.

\paragraph{Competing-destination / local-field model.}
A local destination-field extension (often used to capture spatial competition or agglomeration) is obtained by additionally setting:
\[
C(\ell_j)=\exp(\gamma\ell_j),
\]
while keeping $D(c_{ij})=\exp(-\beta c_{ij})$, $B(o_j)=o_j^\alpha$, and $S(s_{ij})=1$. Then $A_{ij}=o_j^\alpha\exp(-\beta c_{ij})\exp(\gamma\ell_j)$, leading to:
\begin{align}
\frac{g_{ij}}{\tau} &= \alpha\log o_j - \beta c_{ij} + \gamma\ell_j.
\end{align}
The sign of $\gamma$ determines whether the local field acts as an agglomeration benefit ($\gamma>0$) or competition pressure ($\gamma<0$).

\paragraph{Intervening-opportunity model.}
An intervening-opportunity form can be approximated by choosing:
\begin{align*}
&S(s_{ij})=\exp(-\eta s_{ij}), \ 
B(o_j)=1-\exp(-\eta o_j),\\
&D(c_{ij})=1,\quad C(\ell_j)=1,   
\end{align*}

with $\eta>0$. This gives $A_{ij} = \exp(-\eta s_{ij})\left(1-\exp(-\eta o_j)\right)$, and:
\begin{align}
\frac{g_{ij}}{\tau} &= \log\left(1-\exp(-\eta o_j)\right) - \eta s_{ij}.
\end{align}
Here, the destination receives higher intensity as its own opportunity $o_j$ increases (with diminishing returns), but lower intensity as the volume of intervening opportunities $s_{ij}$ grows.

\paragraph{Radiation model.}
Let $M_i$ denote the origin-side opportunity or population mass used in the radiation law (distinct from the learned origin-context representation $m_i$). The radiation capture term, up to an origin-specific production scale that cancels out in the softmax, can be written as:
\[
A_{ij} = \frac{o_j}{(M_i+s_{ij})(M_i+s_{ij}+o_j)}.
\]
Taking the logarithm, we obtain:
\begin{align}
\frac{g_{ij}}{\tau} &= \log o_j - \log(M_i+s_{ij}) - \log(M_i+s_{ij}+o_j).
\end{align}
This represents a threshold-based opportunity model where the choice intensity depends non-linearly on both the destination opportunities and the cumulative opportunities surrounding the origin-destination pair.

\paragraph{Relation to OpFlow.}
These derivations illustrate that classical spatial interaction laws correspond to fixed, low-dimensional, and parametric log-potential forms. OpFlow fundamentally generalizes them by replacing these handcrafted parametric components with shared neural operator channels, while allowing the origin context to dynamically modulate their response strengths. The classical laws are exactly recovered as special cases when the origin gates are fixed to constants, the residual interactions are removed, and the neural channel transforms degenerate into the parametric forms defined above.

\section{Risk Consequences and OOD Generalization}
\label{app:risk-ood}

This section proves the out-of-distribution (OOD) allocation bound and the resulting count-level decomposition used in the OOD Generalization Bounds subsection. We suppress the environment superscript when no ambiguity arises.

\subsection{Proof of Proposition~\ref{prop:ood-drift}}

\begin{proof}
Recall that the true and baseline allocation probabilities are given by
\begin{align}
\pi_{0,i} &= \operatorname{softmax}\left(H_iF^\star(Z_i)/\tau\right), \\
\pi_i^\star &= \operatorname{softmax}\left((H_iF^\star(Z_i)+\Delta_i)/\tau\right),
\end{align}
and the predicted allocation is
\begin{equation}
\widehat\pi_i^\theta = \operatorname{softmax}\left(H_iF_\theta(Z_i)/\tau\right).
\end{equation}
By the triangle inequality for the $L_1$ norm,
\begin{equation}
\left\|\pi_i^\star-\widehat\pi_i^\theta\right\|_1
\le
\left\|\pi_{0,i}-\widehat\pi_i^\theta\right\|_1
+
\left\|\pi_i^\star-\pi_{0,i}\right\|_1.
\end{equation}

We first bound the transferable allocation error. By Pinsker's inequality,
\begin{equation}
\left\|\pi_{0,i}-\widehat\pi_i^\theta\right\|_1
\le
\sqrt{
2\operatorname{KL}\left(\pi_{0,i} \,\|\, \widehat\pi_i^\theta\right)
}.
\end{equation}
Taking the expectation with respect to the test distribution $P_{\rm te}$, applying Jensen's inequality to the concave square root function, and using the density-ratio condition $\frac{dP_{\rm te}}{dP_{\rm tr}}\le \rho$, we obtain
\begin{align}
\mathbb{E}_{\rm te}
\left[
\left\|\pi_{0,i}-\widehat\pi_i^\theta\right\|_1
\right]
&\le
\sqrt{
2\mathbb{E}_{\rm te}
\left[
\operatorname{KL}\left(\pi_{0,i} \,\|\, \widehat\pi_i^\theta\right)
\right]
} \notag \\
&\le
\sqrt{
2\rho
\mathbb{E}_{\rm tr}
\left[
\operatorname{KL}\left(\pi_{0,i} \,\|\, \widehat\pi_i^\theta\right)
\right]
} \notag \\
&=
\sqrt{2\rho R_{\rm tr}^{\rm alloc}(\theta)}.
\end{align}

We next bound the residual drift term. Let $v_i=H_iF^\star(Z_i)$. Then
\begin{equation}
\pi_{0,i}=\operatorname{softmax}(v_i/\tau),
\quad
\pi_i^\star=\operatorname{softmax}((v_i+\Delta_i)/\tau).
\end{equation}
The softmax mapping is Lipschitz continuous from logits to probabilities; in particular, its Lipschitz constant with respect to the $L_2$ norm is bounded by $1$, yielding
\begin{equation}
\left\|
\operatorname{softmax}(x)-\operatorname{softmax}(y)
\right\|_2
\le
\left\|x-y\right\|_2.
\end{equation}
Therefore, using the norm equivalence $\|z\|_1 \le \sqrt{n}\|z\|_2$ for $z \in \mathbb{R}^n$ with $n=|\mathcal{D}_i|$, we have
\begin{align}
\left\|\pi_i^\star-\pi_{0,i}\right\|_1
&\le
\sqrt{|\mathcal{D}_i|}
\left\|\pi_i^\star-\pi_{0,i}\right\|_2 \notag \\
&\le
\frac{\sqrt{|\mathcal{D}_i|}}{\tau}
\left\|\Delta_i\right\|_2.
\end{align}
Taking the expectation with respect to $P_{\rm te}$ and combining the two bounds gives
\begin{equation}
\mathbb{E}_{\rm te}
\left[
\left\|\pi_i^\star-\widehat\pi_i^\theta\right\|_1
\right]
\le
\sqrt{2\rho R_{\rm tr}^{\rm alloc}(\theta)}
+
\mathbb{E}_{\rm te}
\left[
\frac{\sqrt{|\mathcal{D}_i|}}{\tau}
\left\|\Delta_i\right\|_2
\right],
\end{equation}
which completes the proof of Proposition~\ref{prop:ood-drift}.
\end{proof}

\subsection{Count-Level Consequence}

\begin{proposition}[Scale-allocation count decomposition]
\label{prop:count-decomp-app}
Let the predicted and true expected counts be defined as
\begin{equation}
\widehat\lambda_i=\widehat R_i\widehat\pi_i^\theta,
\qquad
\lambda_i^\star=R_i^\star\pi_i^\star,
\end{equation}
where $\widehat\pi_i^\theta$ and $\pi_i^\star$ are probability vectors (i.e., $\|\widehat\pi_i^\theta\|_1 = \|\pi_i^\star\|_1 = 1$). If the true production scale is bounded by $R_i^\star\le R_{\max}$, then
\begin{equation}
\mathbb{E}_{\rm te}
\left[
\left\|
\widehat\lambda_i-\lambda_i^\star
\right\|_1
\right]
\le
\mathcal{R}_{\rm te}^{\rm scale}
+
R_{\max}
\mathbb{E}_{\rm te}
\left[
\left\|
\widehat\pi_i^\theta-\pi_i^\star
\right\|_1
\right],
\end{equation}
where the scale prediction risk is defined as
\begin{equation}
\mathcal{R}_{\rm te}^{\rm scale}
=
\mathbb{E}_{\rm te}
\left[
\left|\widehat R_i-R_i^\star\right|
\right].
\end{equation}
\end{proposition}

\begin{proof}
For each origin $i$, we can decompose the $L_1$ error of the counts by adding and subtracting $R_i^\star\widehat\pi_i^\theta$:
\begin{align}
\left\|
\widehat\lambda_i-\lambda_i^\star
\right\|_1
&=
\left\|
\widehat R_i\widehat\pi_i^\theta - R_i^\star\pi_i^\star
\right\|_1 \notag \\
&\le
\left\|
\widehat R_i\widehat\pi_i^\theta - R_i^\star\widehat\pi_i^\theta
\right\|_1
+
\left\|
R_i^\star\widehat\pi_i^\theta - R_i^\star\pi_i^\star
\right\|_1 \notag \\
&=
\left|\widehat R_i-R_i^\star\right|
\left\|\widehat\pi_i^\theta\right\|_1
+
R_i^\star
\left\|
\widehat\pi_i^\theta-\pi_i^\star
\right\|_1.
\end{align}
Since $\widehat\pi_i^\theta$ is a probability vector, $\|\widehat\pi_i^\theta\|_1=1$. Using the upper bound $R_i^\star\le R_{\max}$, we obtain
\begin{equation}
\left\|
\widehat\lambda_i-\lambda_i^\star
\right\|_1
\le
\left|\widehat R_i-R_i^\star\right|
+
R_{\max}
\left\|
\widehat\pi_i^\theta-\pi_i^\star
\right\|_1.
\end{equation}
Taking the expectation with respect to $P_{\rm te}$ on both sides yields the desired result.
\end{proof}

\paragraph{Interpretation.}
The first term in Proposition~\ref{prop:ood-drift} represents the \emph{transferable allocation error}: it is controlled by the in-distribution training allocation risk and the coverage of the deployment exposure distribution by the training distribution (quantified by the density ratio bound $\rho$). The second term captures the \emph{residual structural drift} that cannot be explained by the observed exposure state. Proposition~\ref{prop:count-decomp-app} further demonstrates that the count-level OOD error additionally depends on the origin-scale prediction error. Consequently, achieving robust OD prediction under distribution shifts requires both a transferable allocation mechanism and a separately calibrated production-scale model.

\section{Supplements to the experiment}

\subsection{Metrics and Baselines}
\label{app:metrics-baselines}

\subsubsection{Metrics}
Let $\mathcal{I}_e$ denote the evaluated origins in environment $e$, and $\mathcal{D}_i^e$ the candidate destination set for origin $i$. For methods predicting counts $\widehat{\lambda}_{ij}^e$, the induced allocation is obtained via row-normalization: $\widehat{\pi}_{ij}^e = \widehat{\lambda}_{ij}^e / \sum_{k \in \mathcal{D}_i^e} \widehat{\lambda}_{ik}^e$. For methods outputting allocations directly, $\widehat{\pi}_i^e$ is used. In synthetic experiments, we evaluate against the oracle allocation $\pi_i^{e,\star}$ and oracle row-centered potential $\psi_i^{e,\star}$. In real-world settings, the evaluation target is the row-normalized observed OD flow.

\paragraph{Allocation Metrics.}
We measure allocation recovery using row-wise KL divergence (Row-KL) and row-wise Common Part of Commuters (Row-CPC) \citep{lenormand2012universal}:
\begin{align}
\mathrm{RowKL} &= \frac{1}{|\mathcal{I}_e|} \sum_{i \in \mathcal{I}_e} \sum_{j \in \mathcal{D}_i^e} \pi_{ij}^e \log \frac{\pi_{ij}^e}{\widehat{\pi}_{ij}^e}, \\
\mathrm{RowCPC} &= \frac{1}{|\mathcal{I}_e|} \sum_{i \in \mathcal{I}_e} \sum_{j \in \mathcal{D}_i^e} \min\left(\pi_{ij}^e, \widehat{\pi}_{ij}^e\right).
\end{align}
Since both vectors are probability distributions, Row-CPC equals $1 - \frac{1}{2}\|\pi_i^e - \widehat{\pi}_i^e\|_1$. Lower Row-KL and higher Row-CPC indicate better allocation recovery.

\paragraph{Mechanism Metric (Synthetic).}
To evaluate whether the model recovers the identifiable choice potential rather than merely matching the induced allocation, we report Potential MSE (PMSE):
\begin{align*}
\mathrm{PMSE}
&= \frac{1}{|\mathcal{I}_e|}
\sum_{i \in \mathcal{I}_e}
\frac{1}{|\mathcal{D}_i^e|}
\left\| \widehat{\psi}_i^e - \psi_i^{e,\star} \right\|_2^2, \notag \\
&\text{where } \widehat{\psi}_i^e = H_i u_i^e.
\end{align*}

\paragraph{Count Metrics.}
To evaluate reconstructed OD flows over pairs $\mathcal{P}_e = \{(i,j) : i \in \mathcal{I}_e, j \in \mathcal{D}_i^e\}$, we report standard RMSE and MAE. For the highly long-tailed real-world counts, we additionally report Log-RMSE:
\begin{equation*}
\begin{aligned}
\mathrm{LogRMSE}
=
\bigg[
&\frac{1}{|\mathcal{P}_e|}
\sum_{(i,j) \in \mathcal{P}_e}
\Big(
\log(1+\widehat{\lambda}_{ij}^e) \\
&\quad - \log(1+Y_{ij}^e)
\Big)^2
\bigg]^{1/2}.
\end{aligned}
\end{equation*}
While allocation metrics measure \emph{where} demand goes, count metrics additionally evaluate production-scale calibration.

\subsubsection{Synthetic Baselines}

The synthetic experiments are designed to isolate target-level, mechanism-level, and OOD-generalization effects. We compare OpFlow against the following groups (all results averaged over 30 random seeds):

\begin{itemize}
    \item \textbf{Supervision Probes:} \textsc{RawCountMLP} (a pairwise MLP trained on raw counts to test if supervision is dominated by origin-scale variation), \textsc{PairAllocMLP} (the same architecture trained on row-normalized allocations to isolate allocation-level supervision), and \textsc{ScaleAwareMLP} (predicts scale and allocation separately without structured exposure operators).
    \item \textbf{Fixed-Share \& Classical:} \textsc{ShareStable} assumes allocation shares are strictly invariant across environments. \textsc{BestClassical} selects the best-performing (via validation Row-KL) among restricted spatial-interaction laws (Gravity, Intervening-Opportunity, Radiation).
    \item \textbf{Domain Generalization (DG):} Standard DG objectives (\textsc{GroupDRO}, \textsc{VREx}, \textsc{IRM}, \textsc{DANN}) applied to the \textsc{PairAllocMLP} backbone. \textsc{BestDG} reports the best validation-selected DG method.
    \item \textbf{OpFlow Variants:} Component ablations that remove specific mechanism channels to identify their structural contributions under opportunity redistribution and exposure-correlation shifts.
\end{itemize}

\subsubsection{Real-World Baselines}

For the real-world cross-county benchmark, we compare against three families of baselines. All methods are trained and validated on the same environments, and evaluated using Row-KL, Row-CPC, and Log-RMSE across the RandSP, GeoSP, and OppSP splits.

\begin{itemize}
    \item \textbf{Classical Models:} Gravity, Intervening-Opportunity, and Radiation models. These provide interpretable geographic priors but lack the capacity to model nonlinear cross-county heterogeneity.
    \item \textbf{Generic Machine Learning:} Off-the-shelf predictors (RF, GBRT, GM) and neural baselines (\textsc{PairMLP}, \textsc{ScaleMLP}). These test whether generic nonlinear prediction or explicit scale awareness suffices without opportunity-conditioned exposure structures.
    \item \textbf{Deep \& Graph OD Models:} Representative deep, graph, and generative baselines (DGM, NetGAN, Diff, WeDAN, GMEL). While capturing rich OD dependencies, they do not explicitly decouple origin production scale from destination allocation via row-centered potentials.
\end{itemize}



\end{document}